\documentclass{article}

\def\inclapp{0}
\def\viewchanges{1}
\def\viewauthors{1}
\def\viewkeywords{0}
\def\usehyperlinks{1}
\def\addackn{0}
\def\preprint{1}

\usepackage[dvipsnames]{xcolor}

\if\preprint1
    \usepackage[preprint]{tmlr}
\else
    \if\viewauthors1
        \usepackage[accepted]{tmlr}
    \else
        \usepackage{tmlr}
    \fi
\fi

\newcommand{\myAND}{\\}
\let\myAND\AND
\let\AND\undefined

\usepackage{natbib}

\usepackage[utf8]{inputenc} 
\usepackage[T1]{fontenc}    
\usepackage{booktabs}       
\usepackage{amsfonts}       
\usepackage{nicefrac}       
\usepackage{microtype}      
\usepackage{amsthm}
\usepackage{amsmath}
\usepackage{amssymb}
\usepackage{mathtools}
\usepackage{enumitem}
\usepackage{mathrsfs}
\usepackage{bbm}
\usepackage{url}            

\usepackage{breakurl}
\if\usehyperlinks1
	\usepackage[colorlinks, citecolor=blue, linkcolor=magenta, breaklinks]{hyperref}       
\fi

\usepackage{graphicx,wrapfig,lipsum}
\usepackage{subfigure}
\usepackage{verbatim}
\usepackage{bm}
\usepackage{adjustbox}
\usepackage{footnote}
\usepackage{algorithm}
\usepackage{algorithmic}
\usepackage{tikz}
\usepackage{cleveref}
\usepackage{multirow}

\if\inclapp0
	\usepackage{xr}
\fi

\newtheorem{theorem}{Theorem}[section]
\newtheorem{lemma}[theorem]{Lemma}
\newtheorem{definition}[theorem]{Definition}
\newtheorem{lem}[theorem]{Lemma}
\newtheorem{prop}[theorem]{Proposition}
\newtheorem{rem}[theorem]{Remark}
\newtheorem{cor}[theorem]{Corollary}
\newtheorem{example}[theorem]{Example}
\newtheorem{assumption}{Assumption}
\newtheorem{AssumptionAppendix}[theorem]{Assumption}

\crefname{assumption}{assumption}{assumptions}
\Crefname{assumption}{Assumption}{Assumptions}
\crefname{lem}{lemma}{lemmata}
\Crefname{lem}{Lemma}{Lemmata}

\let\P\undefined%
\newcommand{\1}{\mathbbm{1}}
\newcommand{\P}{\mathbb{P}}
\newcommand{\F}{\mathcal{F}}
\newcommand{\A}{\mathcal{A}}
\newcommand{\E}{\mathbb{E}} 
\newcommand{\N}{\mathbb{N}} 
\newcommand{\R}{\mathbb{R}} 
\newcommand{\C}{\mathcal{C}} 
\newcommand{\argmin}{\operatorname{argmin}}

\newcommand{\proj}[1]{\operatorname{proj}_{#1}}

\newcommand{\dd}{\,\mathrm d}

\let\del\undefined
\let\com\undefined

\if\viewchanges1
	
	\newcommand{\del}[1]{{\color{red}{#1}}}
	\newcommand{\com}[1]{{\color{orange}{#1}}}
	
\else
	
	\newcommand{\del}[1]{}
	\newcommand{\com}[1]{}
	
\fi

\allowdisplaybreaks

\title{{Neural Jump ODEs as Generative Models}}

\if\viewauthors1
	\author{%
        \name Robert A. Crowell \email{robert.crowell@nyu.edu} \\
        \addr Department of Finance and Risk Engineering\\
        New York University
        \myAND 
        \name Florian Krach \email{florian.krach@me.com} \\
        \addr Department of Mathematics\\
        ETH Zurich
        \myAND 
         \name Josef Teichmann \email{josef.teichmann@math.ethz.ch} \\
        \addr Department of Mathematics\\
        ETH Zurich
	}
\else
	\author{}
\fi

\providecommand{\keywords}[1]{\textbf{{Keywords:}} \textit{#1}}

\begin{document}

\maketitle

\begin{abstract}
In this work, we explore how Neural Jump ODEs \citep[NJODEs;][]{krach2022optimal} can be used as generative models for It\^o processes. Given (discrete observations of) samples of a fixed underlying It\^o process, the NJODE framework can be used to approximate the drift and diffusion coefficients of the process. Under standard regularity assumptions on the It\^o processes, we prove that, in the limit, we recover the true parameters with our approximation. Hence, using these learned coefficients to sample from the corresponding It\^o process generates, in the limit, samples with the same law as the true underlying process.
Compared to other generative machine learning models, our approach has the advantage that it does not need adversarial training and can be trained solely as a predictive model on the observed samples without the need to generate any samples during training to empirically approximate the distribution. 
Moreover, the NJODE framework naturally deals with irregularly sampled data with missing values as well as with path-dependent dynamics, allowing to apply this approach in real-world settings. In particular, in the case of path-dependent coefficients of the It\^o processes, the NJODE learns their optimal approximation given the past observations and therefore allows generating new paths conditionally on discrete, irregular, and incomplete past observations in an optimal way. 



\end{abstract}

\if\viewkeywords1
	\keywords{}
\fi

\section{Introduction}\label{sec:Introduction}

In this work, we consider a potentially path-dependent It\^o process, i.e., a stochastic process $X = (X_t)_{t \in [0,T]}$ solving  the $d$-dimensional SDE
\begin{equation}
    dX_t = \mu_t(X_{\cdot \wedge t})\dd t +  \sigma_t(X_{\cdot \wedge t})\dd W_t\,,
\end{equation}
where $W$ is a $m$-dimensional Brownian motion and $\mu, \sigma$ are the drift and diffusion coefficients, taking values in $\R^d$ and $\R^{d \times m}$ respectively. We assume $\mu$,  $\sigma$ to be fixed but unknown. 
Given a training set with discrete observations of independent samples of this process, our objective is to generate new independent trajectories of $X$. By  learning approximations $\hat\mu, \hat\sigma$ of the true coefficients $\mu$, $\sigma$ we can generate samples having the same law as $X$, provided our approximations are exact. 

\paragraph{Estimating Coefficients}
Thus the key element in this generative model approach is learning to approximate the coefficients $\mu, \sigma$. 
To do this, we use the Neural Jump ODE (NJODE) framework, which was first introduced in \citet{herrera2021neural} and then refined and extended several times in \citet{krach2022optimal, NJODE3, krach2024learning, heiss2024nonparametric}.
The NJODE is a model that allows to optimally predict stochastic processes in continuous-time. In this model, the underlying stochastic processes can be path-dependent and may have jumps (for simplicity we restrict to continuous trajectories). The predictions are based on discrete observations of the past, which may be irregular and incomplete. These observations generate the $\sigma$-algebras $\mathcal{A}_t$, which encode the currently available information at any time $t\in [0,T]$.
Theoretical guarantees show that the NJODE converges to the optimal prediction in the $L^2$-sense, meaning that the NJODE reconstructs the conditional expectation $\E[X_t | \mathcal{A}_s]$, for any $s \leq t$ as a limiting object. Similarly, the NJODE can, for example, be applied to moments of the process to learn to predict $\E[X_t X_t^\top | \mathcal{A}_s]$. 
These predictions can be used to estimate the coefficients of the It\^o process in the following way. 
For this, let us assume that NJODE models have already been trained on some training set such that they approximate conditional expectations $(s,t) \mapsto \E[X_t | \mathcal{A}_s]$ and $(s,t) \mapsto \E[X_t X_t^\top | \mathcal{A}_s]$ arbitrarily well. Then for any fixed time $t$, we can use the Euler-Maruyama scheme to discretise the next step of the process $X$ with a time step $\Delta > 0$ and corresponding independent Brownian increment $\Delta W_t \sim N(0, \Delta)$ as 
\begin{equation*}
    X_{t + \Delta} \approx X_t +  \mu_t \Delta + \sigma_t \Delta W_t.
\end{equation*}
Assuming that $X_t$ was observed, i.e., $X_t \in \mathcal{A}_t$, we can apply the conditional expectation on both sides to get
\begin{equation*}
    \E[X_{t + \Delta} | \mathcal{A}_t ] \approx X_t +  \E[\mu_t | \mathcal{A}_t ] \Delta,
\end{equation*}
which can be rearranged to the following estimator of $\mu$
\begin{equation}\label{eq:hat mu}
    \hat{\mu}_t^{\Delta} := \frac{\E[X_{t + \Delta} | \mathcal{A}_t ] - X_t}{\Delta} \approx \E[\mu_t | \mathcal{A}_t ].
\end{equation}
We note that the estimator $\hat \mu_t^{\Delta}$ can be expressed through the available information ($X_t$) together with the NJODEs approximation of the conditional expectation of $X$.
If $\mu_t$ is measurable with respect to the known information, i.e., $\mu_t \in \mathcal{A}_t$, then the RHS of \eqref{eq:hat mu} simplifies to $\E[\mu_t | \mathcal{A}_t ] = \mu_t$. Otherwise, $\hat \mu_t^{\Delta}$ is an estimator for the $L^2$-optimal approximation $\E[\mu_t | \mathcal{A}_t ]$ of $\mu_t$ given the available information. 

Applying first It\^o's formula to the components\footnote{%
For readability within the Introduction, we denote by $X^i$, $\mu^i$ the $i$-th element of the respective vectors and by $\sigma^i$ the $i$-th row of the matrix, for $1\leq i \leq d$. Moreover, for a matrix $M$ we denote the $(i,j)$-th element as $M_{i,j}$ and write $M = (M_{i,j})_{i,j}$. Vectors are, by default, assumed to be column vectors. Later, we will write the coordinate index as subscript.%
} %
of $X X^\top = (X^i X^j)_{i,j}$ and discretizing the resulting SDE for a $\Delta$-step with Euler-Maruyama as before, we get
\begin{equation*}
(X^iX^j)_{t+\Delta} \approx (X^iX^j)_t + \mu^i_tX^j_t\Delta+X^j_t \sigma_t^i\Delta W_t+X^i_t\mu^j_t \Delta+ X^i \sigma^j_t \Delta W_t+\sigma^i_t(\sigma^j_t)^\top \Delta.
\end{equation*} 
Taking the conditional expectation, using $\E[\mu_t | \mathcal{A}_t ] \approx \hat \mu_t^{\Delta}$ and rearranging, we get the estimator of $\Sigma_t^{\Delta} := \sigma_t^{\Delta} (\sigma_t^{\Delta})^\top$,
\begin{equation}\label{eq:hat Sigma}
    (\hat{\Sigma}_t^{\Delta})_{i,j} := \frac{\E[(X^i X^j)_{t + \Delta} | \mathcal{A}_t] - (X^i X^j)_t}{\Delta} - X^i_t \hat{\mu}^{\Delta,j}_t - X^j_t \hat{\mu}^{\Delta,i}_t \approx \E[ (\Sigma_t)_{i,j} | \mathcal{A}_t] .
\end{equation}
This estimator can be expressed through the NJODE approximation of the conditional expectation of $X$ and its moments, as well as the current information. Again, if $\sigma_t \in \mathcal{A}_t$, then the RHS of \eqref{eq:hat Sigma} simplifies to $\E[\Sigma_t | \mathcal{A}_t ] = \Sigma_t$. Otherwise, $\hat \Sigma_t^{\Delta}$ is an estimator for the $L^2$-optimal approximation $\E[\Sigma_t | \mathcal{A}_t ]$ of $\Sigma_t$ given the available information. 

\paragraph{A better estimator for $\Sigma$}
By definition, the instantaneous variance matrix $\Sigma_t$ is symmetric and positive semi-definite. These properties also hold for $\E[ \Sigma_t | \mathcal{A}_t]$, since by the linearity of the expectation we have for any fixed vector $v \in \R^d$ that 
\begin{equation*}
    v^\top \E[ \Sigma_t | \mathcal{A}_t] v = \E[v^\top  \Sigma_t v| \mathcal{A}_t] = \E[ \lvert \sigma_t^\top v \rvert_2^2\, |\, \mathcal{A}_t] \geq 0,
\end{equation*}
showing positive semi-definiteness (and symmetry is trivially true).
However, the estimator $\hat \Sigma^{\Delta}$, as defined in \eqref{eq:hat Sigma}, might not satisfy these properties due to numerical errors in the estimation of the individual components.
This is problematic, since then we cannot find a matrix square root of it, which we need for the generation of samples.
Therefore, we suggest to rewrite the estimator, using the definition of $\hat \mu^{\Delta}$, as
\begin{equation}\label{eq:hat Sigma 2}
    \hat{\Sigma}_t^{\Delta} := \frac{ 1}{\Delta} \E[(X_{t + \Delta} - X_t)(X_{t + \Delta} - X_t)^\top | \mathcal{A}_t],
\end{equation}
which satisfies the properties by definition.
To compute this estimator with the NJODE framework, we define the \emph{squared increments} process 
\begin{equation*}
    Z_{t} := (X_t - X_{\tau(t)})(X_t - X_{\tau(t)})^\top, \quad 0 \leq t \leq T,
\end{equation*} 
where $\tau(t)$ is the last observation time before time $t$\footnote{This definition only makes sense if $X$ has complete observations and needs to be adapted accordingly by taking the last coordinate-wise observation time instead of $\tau(t)$.}.
Then, by training a NJODE to predict $Z$ using the generalized training framework of \citet{krach2024learning}, we learn to approximate \eqref{eq:hat Sigma 2} up to a known constant factor (by evaluating it at $\Delta$ after the last observation time). A priori, the NJODE output predicting $Z$ does not necessarily satisfy the symmetry and positive semi-definiteness properties though. However, denoting the output of the NJODE model as $G \in \R^{d\times d}$, we can define $S=GG^\top$ and train $S$ instead of $G$ to predict $Z$. Then the necessary properties are hardcoded into $S$ by its definition and we directly have access to a square root of $S$ given by the model output $G$. 

We note that the NJODE model predicting $Z$ needs to get the information of $\mathcal{A}_t$ as input, i.e., the observations of $X$ and not (only) the observations of $Z$ (which have less information). The example of a geometric Brownian motion, where $\Sigma_t$ depends on the value $X_t$, exemplifies that using observations of $Z_t$ as input is not sufficient. Hence, this puts us in an input-output setting \citep{heiss2024nonparametric}, where $Z$ is the output process learned from the input processes $X,Z$ (theoretically, $X$ would suffice as input process, but additionally using $Z$ can simplify the learning). 

\paragraph{Instantaneous parameter estimation} The estimators heuristically derived above are natural to study and unbiased in the limit $\Delta\to 0$. However, before passing to the limit, they may be biased. We demonstrate that the NJODE framework is versatile and powerful enough to overcome this bias even without passing to the computationally infeasible limit $\Delta\to 0$. Indeed, in \Cref{sec:Estimating the Instantaneous Coefficients}, we show how with a more involved estimation procedure, we can accurately learn the instantaneous coefficients.

\paragraph{Approximating the law of $X$}
We note that we can only estimate the square of $\sigma$, since the law of $X$ (which we ultimately use through the conditional expectations) is determined by $\Sigma$ regardless of its true square root $\sigma$.
Vice versa, any square root $\hat \sigma^{\Delta}$ of $\hat \Sigma^{\Delta}$ can interchangeably be used to define the SDE
\begin{equation}\label{eq:tilde X SDE}
    \dd \tilde X_t = \hat \mu_t^{\Delta} \dd t + \hat \sigma_t^{\Delta} \dd W_t,
\end{equation}
whose solution $\tilde X$ approximates $X$ in law. Therefore, new samples approximating the distribution of $X$ can be generated by sampling from \eqref{eq:tilde X SDE}. 
As we show in \Cref{thm:thm.2}, in the limit $\Delta \to 0$, the law of $\tilde X$ converges to the law of $X$, under the assumption that $\mu_t, \sigma_t \in \mathcal{A}_t$. 

\paragraph{Sample generation}
In practice, a discretization scheme, like Euler-Maruyama, is used to sample from the SDE \eqref{eq:tilde X SDE}, by iteratively computing the coefficient estimators $\hat \mu, \hat \sigma$ given the past sampled points (or observations), using them to generate the next point and appending this to the generated sequence (and therefore to the available information). This procedure can by used to generate new sequences starting from any given initial point $X_0$ or, alternatively, from any fixed starting sequence $(X_0, X_{t_1}, \dotsc, X_{t_k})$ for observation times $0 < t_1 < \dotsb < t_k < T$.
Moreover, if the starting point or sequence has missing values, the approach naturally extends by first predicting $\E[X_0 | \mathcal{A}_t]$ or $\E[X_{t_k} | \mathcal{A}_{t_k}]$, respectively, as starting point for the further generation of the samples.

\subsection{Related Work}\label{sec:Related Work}
In comparison to many of the standard machine learning approaches for generative models in the context of time series generation \citep[e.g., neural SDEs trained as GANs;][]{kidger2021neural}, our approach has the advantage of being trained in a pure prediction setting, without the need to actually generate samples for the training procedure. This makes our approach more efficient in training. Moreover, the training works by minimizing a well-defined, MSE-type loss function, which admits a unique optimizer (up to indistinguishability). Hence, we can derive theoretical convergence guarantees, implying the convergence of samples from our generative approach to the true distribution. 

In contrast to this, GAN-type approaches for time-series generation \citep{chen2018model,yoon2019time,henry2019generative,wiese2020quant,xu2020cot,cuchiero2020generative,gierjatowicz2020robust,kidger2021neural,cont2022tail,flaig2022scenario,liu2022time,rizzato2023generative} build on two competing players in a zero-sum minimax game, which does not necessarily have Nash equilibria \citep{pmlr-v119-farnia20a}, and even if they exist, convergence to them is not certain \citep{pmlr-v80-mescheder18a}. In line with this, GANs have frequently been reported to fail to converge to a stable solution in practice \citep{pmlr-v80-mescheder18a}.
Typical generative models for time-series generation are neural SDEs, neural diffusion models, and deep conditional step-wise generators. 
Even if such models do not rely on adversarial training \citep{neuralSDE1, remlinger2022conditional,liao2020conditional,buehler2020data,desai2021timevae,huang2024generative,lu2024generative,acciaio2024time,jahn2025trajectory}, their need to generate samples for the training procedure, to empirically approximate the generator's distribution, can make the training more inefficient.
Moreover, models that use expected values of the evaluation of a function of the generated process at certain times in the loss function \citep{neuralSDE1,cuchiero2020generative}, can usually only control marginal distributions of the process but not its entire law, as is the case with our approach.

Similarly to our approach, \citet{cohen2023arbitrage} use a neural SDE, where they directly learn the coefficients without sampling; however, unlike us, their method does not inherently deal with incomplete observations in the training data. Moreover, while they construct the model such that it is arbitrage free, it is not studied whether the model converges to the true law of the underlying data (this is not an objective since it is assumed that only one realization of the underlying process is available, as typical in financial time series).

\subsection{Outline of the Work}\label{sec:Outline}

In this work, we propose a new, fully forecast-based, deep learning generative framework for diffusion processes, as heuristically described in \Cref{sec:Introduction}. The approach consists of two well-separated steps. First, NJODE forecasting models are trained to approximate conditional expectations. For this, the necessary problem setting with details on notation and assumptions is given in \Cref{sec:Problem Setting}. Then the NJODE formulation and its training framework are given in \Cref{sec:The Neural Jump ODE Model for Coefficient Estimation}. In \Cref{sec:Details Assumptions and Theoretical Guarantees for the Coefficient Estimates} the idealized coefficient estimators are defined, and trained NJODE forecasting models are used to define realizable coefficient estimators, which are proven to converge to the true coefficients. 
In \Cref{sec:Estimating the Instantaneous Coefficients}, the NJODE method is further refined to directly estimate the instantaneous coefficients, which are the limiting objects of the step-wise estimates when the step size goes to $0$.
These estimates (step-wise or instantaneous) are then used to generate samples, whose law is proven to converge to the true distribution in \Cref{sec:The Generative Procedure}.
Experiments showing the applicability of this approach are presented in \Cref{sec:Experiments}.

\section{Problem Setting}\label{sec:Problem Setting}
We build on previous work on NJODEs and therefore follow their setting, particulalry those of \citet{krach2022optimal, heiss2024nonparametric, krach2025operator}.

\subsection{Stochastic Process, Random Observation Times and Observation Mask}
\label{sec:Stochastic Process, Random Observation Times and Observation Mask}
Let $d\in \N$ be the dimension and $T>0$ be a fixed time horizon. We work on a filtered probability space $(\Omega, \mathcal{F}, \mathbb{F},\mathbb{P})$, where the filtration $\mathbb{F}=(\mathcal{F}_t)_{t \in \mathbb{R}_+}$ satisfies the usual conditions, i.e.\ , the $\sigma$-field $\mathcal{F}$ is $\mathbb{P}$-complete, $\mathbb{F}$ is right-continuous and $\mathcal{F}_0$ contains all $\mathbb{P}$-null sets of $\mathcal{F}$.  
On this filtered probability space, we consider an adapted, $d$-dimensional, continuous stochastic process $X :={(X_t)}_{t \in [0,T]}$, which satisfies the following assumption.
\begin{assumption}\label{ass:1}
    The dynamics of the diffusion process $X$ are given by
    \begin{equation}\label{eq:def X}
        X_t = x_0 + \int_0^t \mu_s(X_{\cdot \wedge s})\dd s + \int_0^t \sigma_s(X_{\cdot \wedge s})\dd W_s,\, \qquad \text{for}~t\in [0, T],
    \end{equation}
    where $\mu$ and $\sigma$ are progressively measurable functionals taking values in $\R^d$ and $\R^{d\times m}$, respectively, which are uniformly bounded and jointly continuous, and $W=(W_t)_{t\in [0, T]}$ is an $m$-dimensional standard Brownian motion. 
\end{assumption}
In this work, we distinguish between the training set, which is used to learn approximating the necessary conditional expectations with NJODEs to get estimates of the coefficients, and the starting sequence and generated data in inference, when the approach is used in a generative way.

\subsection{Information \texorpdfstring{$\sigma$}{sigma}-algebra}\label{sec:information sigma-algebra}
For the training set, we assume that a random number of $n\in \N$ observations take place at the random $\mathbb F$-stopping times
\begin{align}\label{eqn:obs.times}
    0 = t_0 < t_1 < \dotsb < t_{n} \le  T 
\end{align}
and denote by $\bar n=\sup \left\{k \in \N \, | \, \P(k = n) > 0 \right\} \in \N \cup\{\infty\}$ the maximal value of $n$.
Note that this set-up allows for a possibly unbounded number of observations in the finite time interval $[0,T]$.
Moreover, we define the random functions
\begin{equation*}
    \tau(t) :=  \max\{ t_k \colon   t_k \leq t \} ,\qquad
    \kappa(t) := \max\{ k \colon  t_k \leq t \},
\end{equation*}
which denote the last observation time and the number of observation times (or zero if no observation was made yet) before time $t \in [0,T]$.
Observations can have missing values, which is formalised through the observation mask, a sequence of random variables $M = (M_k)_{0 \leq k \leq \bar n}$ taking values in  $\{ 0,1 \}^{d}$. If $M_{k,j}=1$, then  the $j$-th coordinate $X_{t_k,j}$ is observed at observation time $t_k$. By abuse of notation, we  also write $M_{t_k} := M_{k}$ and assume that $M_{t_k} \in \F_{t_k}$. 

The information available at time $t$ is given by the values of the process $X$ at the observation times when not masked, as well as the observation times and masks until $t$. This leads 
to the \emph{filtration of the currently available information} $\mathbb{A} := (\mathcal{A}_t)_{t \in [0, T]}$ given by
\begin{equation*}
\A_t := \boldsymbol{\sigma}\left(X_{t_i, j}, t_i, M_{t_i} \mid t_i \leq t,\, j \in \{1 \leq l \leq d \mid M_{t_i, l} = 1 \} \right) \subseteq \F_t,
\end{equation*}
where $\boldsymbol{\sigma}(\cdot)$ represents the generated $\sigma$-algebra.
By the definition of $\tau$, we have $\mathcal{A}_t = \mathcal{A}_{\tau(t)}$ for all $t \in [0, T]$. 
Additionally,  for any fixed observation (or stopping) time $t_k$, the stopped and pre-stopped\footnote{The stopped $\sigma$-algebra \citep[Definition 2.37]{KarandikarRao2018} is defined as $\mathcal{F}_{\tau} = \{A \in \sigma(\cup_{t} \mathcal{F}_{t}) : A \cap \{\tau \leq t\} \in \mathcal{F}_{t} \, \forall t\}$, where $\tau$ is the stopping time. The pre-stopped $\sigma$-algebra \citep[Definition 8.1]{KarandikarRao2018} is defined as $\mathcal{F}_{\tau-} = \sigma\left(\mathcal{F}_0 \cup \{A \cap \{t < \tau\} : A \in \mathcal{F}_t, \, t < \infty\}\right)$, where $\tau$ is the stopping time.} $\sigma$-algebras at $t_k$ are
\begin{equation*}
\begin{split}
\mathcal{A}_{t_k} &:= \boldsymbol{\sigma}\left(X_{t_i, j}, t_i, M_{t_i} \,\middle|\, i\leq k,\, j \in \{1 \leq l \leq d | M_{t_i, l} = 1  \} \right), \\
\mathcal{A}_{t_k-} &:= \boldsymbol{\sigma}\left(X_{t_i, j}, t_i, M_{t_i}, t_k \,\middle|\, i < k,\, j \in \{1 \leq l \leq d | M_{t_i, l} = 1  \} \right) = \mathcal{A}_{t_{k-1}} \vee \boldsymbol{\sigma}(t_k).
\end{split}
\end{equation*} 
We define the $i$-th observation at time $t_i$ as $O_{i} := (M_{t_i} \odot X_{t_i}, t_i,  M_{i_t}) \in \mathscr O \coloneqq \R^d\times \R \times \R^d$. This gives rise to the \emph{information process} $O:[0, T]\times \Omega \to \mathscr O^{\mathbb N}$ given by 
$$ (t, \omega) \mapsto O_{[0, t]}(\omega) \coloneqq (O_1, \dotsc, O_k, 0, \dotsc) \in \mathscr O^{\N}\,.$$
Since $(t_i)_{i \in \N}$ are $\mathbb F$-stopping times, the process $O=(O_{[0, t]})_{t \in [0,T]}$ is $\mathbb F$-progressive. We then call $\boldsymbol\sigma(O_{[0,t]}) = \mathcal{A}_t$ the \emph{information $\sigma$-algebra}, so that $\mathbb A$ is exactly the filtration of currently available information defined above. This makes $O$ also $\mathbb A$-progressive.

\subsection{Notation and Assumptions}\label{sec:Notation and assumptions}
We are interested in the conditional expectation processes of $X$ given the currently available information, i.e., in the process $(\E[X_t\, |\, \mathcal A_t])_{t \in [0, T]}$. By \citep[Cor.~7.6.8]{cohen2015stochastic}, we can find an $\mathbb A$-progressive modification of this process which we denote by $\hat{X}=(\hat X_t)_{0 \leq t \leq T}$, and which satisfies 
\begin{equation*}
    \hat{X}_t := \E[X_t\, | \,\A_t].    
\end{equation*}

Since $\hat X$ and $O$ are $\mathbb A$-progressive, the Doob-Dynkin lemma \citep[Lemma~1.14]{kallenberg2021foundations} implies that there exists a measurable map 
\begin{equation*}
    F^X : [0,T]  \times \mathscr O^{\N} \to \R^{d}\,, \qquad (t, o) \mapsto F^X(t,o):= F^X_{t}(o),
\end{equation*}
satisfying $\hat{X}_{t} = F^X(t, O_{[0,t]})$. 
Similarly, for the process $\hat{Z}=(\hat Z_t)_{0 \leq t \leq T}$ defined via $\hat{Z}_t := \E[Z_t \,|\, \A_t]$, or more specifically as an $\mathbb A$-progressive modification thereof, we define $F^Z$ in the same way.

We make the following assumptions on our framework and denote by $f^X,f^Z$ the generalized time derivatives of $F^X,F^Z$ (see \Cref{sec:Applying the NJODE in the Generative Setting} for more details).

\begin{assumption} \label{assumption:2}
For every $1\leq k, l \leq \bar n$, $M_k$ is independent of $t_l$ and $n$ and 
$ \P ((M_{k, i}) =1 ) > 0$ for every component $1 \leq i \leq d$ of the vector  (every component can be observed at any observation time and point) and $M_0=1$.
\end{assumption} 
\begin{assumption} \label{assumption:3}
The random number of observation times $n$ is integrable, i.e., $\E[n] < \infty$.
\end{assumption} 
\begin{assumption} \label{assumption:4}
The process $X$ is independent of the observation framework, i.e., of the random variables $n, (t_k, M_k)_{k \in \N}$. 
\end{assumption} 
\begin{rem}
    We assume complete observations at $t_0$ to ensure that the process $Z$ is well defined. Generalizations of this assumption are possible, but get more involved.
\end{rem}
\begin{rem}
    The independence \Cref{assumption:2,assumption:4} can be replaced by conditional independence assumptions as formulated in \citet[Section~4]{NJODE3}.
\end{rem}

We use the following (pseudo-)distance functions (based on the observation times) between processes and define indistinguishability with them.
\begin{definition}\label{def:indistinguishability}
Fix $r \in \N$ and set $c_0(k) := (\P (n \geq k))^{-1}$.
The family of (pseudo) metrics $d_k$, $1 \le k \le \bar n$,  for  two c\`adl\`ag $\mathbb{A}$-adapted processes $\eta, \xi  : [0,T] \times \Omega \to \R^{r}$  is defined as 
\begin{equation}\label{equ: pseudo metric}
    d_k (\eta, \xi) = c_0(k)\,  \E\left[ \1_{\{ k \le n\}} \left( | \eta_{t_k-} - \xi_{t_k-} |_1 + | \eta_{t_k} - \xi_{t_k} |_1 \right) \right].
\end{equation}
We call the processes \emph{indistinguishable at observation points}, if $d_k(\eta, \xi)=0$ for every $1 \leq k \leq \bar n$.
\end{definition}

In the following, we show that with \Cref{ass:1}, all necessary conditions are satisfied to apply the NJODE framework to learn to predict the processes $X$ and $Z$ as in \citet{krach2025operator}.
We note that we use \citet{krach2025operator} instead of \citet{heiss2024nonparametric}, since it only requires measurability, but not continuity, of the respective functions.

\begin{prop}\label{prop:NJODE assumptions satisfied}
    If \Cref{ass:1,assumption:2,assumption:3,assumption:4} are satisfied, then the processes $X, Z$ and the observation framework satisfy Assumptions~1 to~7 of \citet{krach2025operator}, hence, the main convergence results for NJODEs \citep[Theorems~4.1 and~4.4]{krach2025operator} can be applied.
\end{prop}
For the proof and more details on the assumptions of \citet{krach2025operator} see \Cref{sec:Applying the NJODE in the Generative Setting}.

\section{The Neural Jump ODE Model for Coefficient Estimation}\label{sec:The Neural Jump ODE Model for Coefficient Estimation}
We use the NJODE model defined in \citet{krach2022optimal,heiss2024nonparametric} to predict the processes $X,Z$ with which we can derive estimators for $\mu,\Sigma$, as outlined in \Cref{sec:Introduction}.
In the following, we give a heuristic overview of the input-output NJODE model and its loss function, while referring to \citet[Definition~3.3]{heiss2024nonparametric} for the exact definition and details.
The  \textit{Input-Output Neural Jump ODE} model is given by
    \begin{equation}\label{equ:PD-NJ-ODE}
\begin{split}
H_0 &= \rho_{\theta_2}\left(0, 0, U_0 \right), \\
dH_t &= f_{\theta_1}\left(H_{t-}, t, \tau(t), U_{\tau(t)}   \right) dt  + \left( \rho_{\theta_2}\left( H_{t-}, t, U_{t}  \right) - H_{t-} \right) dn_t, \\
G_t &=  g_{ \theta_3}(H_t),
\end{split}
\end{equation}
where $U$ is the input process, $n_t$ counts the current number of observations and $G$ is the models output process. The parametric functions $f_{\theta_1},~\rho_{\theta_2}$ and $ g_{ \theta_3} $ are (bounded output) feedforward neural networks with trainable weights $\theta = (\theta_1, \theta_2, \theta_3) \in \Theta$. We write $\Theta_m \subset \Theta$ to denote the compact subset of all possible NN weights that allow the maximum widths and depths (and therefore also the dimension of $H$) to be $m$ and whose norms are bounded by $m$.
For a target output process $V$, we define the theoretical loss function as
\begin{align}
\Psi(V, \eta) &:= \E\left[ \frac{1}{n} \sum_{i=1}^n  \left(  \left\lvert \proj{V} (M_i) \odot ( V_{t_i} - \eta_{t_i} ) \right\rvert_2 + \left\lvert \proj{V} (M_i) \odot (V_{t_i-} - \eta_{t_{i}-} ) \right\rvert_2 \right)^2  \right], \label{equ:Psi} 
\end{align}
where $\odot$ is the element-wise (Hadamard) product and $\proj{V}$ denotes the projection onto the coordinates corresponding to the output variable $V$. 
The empirical loss function $\hat \Psi_N$ is given by the empirical approximation of the expectation with $N$ training samples.

The suggested approach is to use two independent instances of the NJODE model with output processes $G_1^\theta, G_2^\theta$ to predict $X$ and $Z$, respectively, where the model predicting $Z$ must be an input-output model \citep{heiss2024nonparametric} additionally taking $X$ as input, i.e., $U=(X,Z)$, $V=Z$. Since $V$ is a subprocess of $U$, we use the original loss function instead of the IO loss function, following the suggestion in \citet[Section~7.1]{heiss2024nonparametric}.
The process $Z$ has jumps at observation times, which we can deal with by using \citet[Remark~2.4]{krach2025operator}, since we have the left and right limit of the jumps (in particular, the right limit is always $0$ in the case of complete observations)\footnote{%
We now have the observations $Z_{t_k-}$ and $Z_{t_k}$ that can be fed as inputs to the model. To correctly learn the jumps, the model has to get $Z_{t_k}$ as input at the jump. If  $Z_{t_k-}$ doesn't provide additional information, as is the case for us, since the input $X$ already carries all information, then we do not need to feed $Z_{t_k-}$ as input to the model. We use this approach in our implementation. %
}.
As outlined in \Cref{sec:Introduction}, we train the NJODE predicting $Z$ with the loss function
$$\Phi_2(\theta) := \Psi(Z, G_2^\theta (G_2^\theta)^\top),$$ 
while we use the standard loss 
$$\Phi_1(\theta) := \Psi(X, G_1^\theta)$$ 
for the NJODE predicting $X$. In particular, the NJODE prediction for $Z$ is given by $S_2^\theta := G_2^\theta (G_2^\theta)^\top$\footnote{It is a choice of naming, whether one calls $S_2^\theta$ or $G_2^\theta$ the output of the corresponding NJODE model, i.e., whether the squaring is included in (or under the hood of) the model architecture or not. Since the squaring is important to satisfy the constraints (see \Cref{sec:Introduction}), we use the given notation to explicitly state this operation, but we refer to both $G_2^\theta$ and $S_2^\theta$ as model output depending on the context.}.

\begin{rem}\label{rem:long-term prediction suggested}
    In settings where regular and complete observations are available for training the NJODE model, we recommend using the training approach for long-term predictions of \citet{krach2024learning} to get more accurate long-term estimates of drift and variance. Because these estimates are used iteratively to generate the next steps (without insertion of true observations), they have to be accurate over a long time horizon. Otherwise, the generated paths may diverge from the true law, as the prediction of the conditional expectation diverges from the true one \cite[see][Figure~2]{krach2024learning}. Learning long-term predictions should therefore also decrease the short-term errors if small short-term errors blow up in the long run. Additionally, the training approach for long-term predictions leads to a more efficient usage of the training data in the case of regular complete observations, which further reduces the error. 
    Using this generalized training approach, the NJODE models can predict the conditional expectations at any time $t \in [0,T]$ given information up to time $s \leq t$\footnote{This is done by feeding the observations until $s$ as input to the NJODE model and then continuing the prediction until $t$ without any further inputs.}. 
\end{rem}
\begin{definition}
    To simplify the notation, we write $G_{s,t-s}^\theta$ for the NJODE prediction of $\E[X_t | \mathcal{A}_s]$ and $S_{s,t-s}^\theta$ for the NJODE prediction of $\E[Z_{t,s} | \mathcal{A}_s]$, for any $0\leq s \leq t \leq T$. 
    In particular, $G_{s,t-s}^\theta$ corresponds to the output of the first NJODE model $G_1^\theta$ and $S_{s,t-s}^\theta$ corresponds to the output of the second NJODE model $S_2^\theta$ or $G_2^\theta$, respectively.
\end{definition}

\section{Details, Assumptions and Theoretical Guarantees for the Coefficient Estimates}\label{sec:Details Assumptions and Theoretical Guarantees for the Coefficient Estimates}
 
In this section, we give theoretical guarantees for the estimation of coefficients and sample generation with the NJODE model. The generative procedure, briefly described in \Cref{sec:Introduction}, will then be described in \Cref{sec:The Generative Procedure}.
Further details on the model, its implementation and training are given in \Cref{sec:The Neural Jump ODE Model for Coefficient Estimation}.

In the following, we first define the different coefficient estimators. For a given step size $\Delta > 0$, the \emph{idealized} drift and diffusion estimators are
\begin{align}
    \hat{\mu}_t^\Delta &:= \frac{1}{\Delta} E[X_{t + \Delta} - X_t | \mathcal{A}_t ] \label{equ:mu hat Delta idealized}\\
    \hat{\Sigma}_t^{\Delta} &:= \frac{ 1}{\Delta} \E[(X_{t + \Delta} - X_t)(X_{t + \Delta} - X_t)^\top | \mathcal{A}_t]. \label{equ:sigma hat Delta idealized}
\end{align}
These idealized estimators are, in general, not computable, since the necessary conditional expectations are not accessible. Therefore, we use their approximations using the NJODE model (cf.\ \Cref{sec:The Neural Jump ODE Model for Coefficient Estimation}) for the following \emph{realizable NJODE} drift and diffusion estimators
\begin{align}
    \hat{\mu}_t^{\Delta, \theta} &:= \frac{1}{\Delta} (G_{t,\Delta}^\theta - G_{t,0}^\theta), \\
    \hat{\Sigma}_t^{\Delta, \theta} &:= \frac{ 1}{\Delta} S_{t,\Delta}^\theta.
\end{align}
Note that in case $X_t \in \mathcal{A}_t$, the estimate $G_{t,0}^\theta$ can be replaced by $X_t$ to recover the standard estimate for \eqref{eq:hat mu} as described in \Cref{sec:Introduction}. This corresponds to the standard situation during iterative generation, with the only possible exception at the starting point, since afterwards complete samples are generated and therefore used as inputs for the next generation step.

To show convergence of the NJODE estimators, we use model parameters $\theta$ minimizing the loss functions. For ease of notation, we do not explicitly distinguish between the parameters of $G_1^\theta$ and $S_2^\theta$ and simply write $\Theta^{\min}_{m,N} = \argmin_{\theta \in \Theta_m} \{ \hat\Psi_N(\theta) \}$ implicitly deciding between the parameters and corresponding (empirical) objective functions for the two NJODE models.
In practice, the models can either be trained independently, or one can also define one joint model that provides both outputs $G_1^\theta$ and $S_2^\theta$ and jointly train them by using the different loss functions for the respective outputs.

\begin{rem}\label{rem:self-injected bias correction}
    Training a joint model has the additional advantage that it can facilitate a self-injected bias reduction for the diffusion estimator.
    In particular, the increment $X_{t}-X_{\tau(t)}$ is, in general, not conditionally unbiased, and the bias $\E[X_{t}-X_{\tau(t)} | \mathcal{A}_{\tau(t)}]$ often increases with $\Delta = t - {\tau(t)}$. Hence, after squaring the increment, this bias term can contribute a substantial part to the value of $\E[Z_{t} | \mathcal{A}_{\tau(t)}]$. The larger the range of the target values (in this case $Z$, which takes the value $0$ at observation times), the less precise the predictions are in absolute terms, since an error of $\varepsilon$ has less impact on the total value of the loss. 
    Using the bias-corrected increments
    \begin{equation*}
        (X_{t}-X_{\tau(t)}) - \E[X_{t}-X_{\tau(t)} | \mathcal{A}_{\tau(t)}] = X_{t} - \E[X_{t}| \mathcal{A}_{\tau(t)}] 
    \end{equation*}
    to define the \emph{quadratic bias-corrected increments process}
    \begin{equation*}
        Z^{BC}_t = (X_t - E[X_t|\mathcal{A}_{\tau(t)}]) (X_t - E[X_t|\mathcal{A}_{\tau(t)}])^\top,
    \end{equation*}
    we can lower the values of the corresponding conditional expectation $\E[Z^{BC}_{t} | \mathcal{A}_{\tau(t)}]$. This conditional expectation coincides with the conditional covariance of $X$ and of its increment process 
    \begin{equation*}
        \E[Z^{BC}_{t} | \mathcal{A}_{\tau(t)}] = \operatorname{Var}[X_{t} | \mathcal{A}_{\tau(t)}] = \operatorname{Var}[X_{t} - X_{\tau(t)} | \mathcal{A}_{\tau(t)}],
    \end{equation*}
    while the conditional expectation of $Z$ corresponds to the (strictly larger) second moment of the increment process.
    Since we do not have access to $E[X_t|\mathcal{A}_{\tau(t)}]$, we cannot use the process $Z^{BC}$ directly as target for training the NJODE output $S_2^\theta$. However, the NJODE output $G_1^\theta$ approximates $E[X_t|\mathcal{A}_{\tau(t)}]$, therefore we can instead use 
    \begin{equation*}
        \tilde Z^{BC}_t = (X_t - (G^\theta_1)_t) (X_t - (G^\theta_1)_t)^\top,
    \end{equation*}
    as target for training $S_2^\theta$. By jointly training $G^\theta_1, S_2^\theta$, this leads to a self-injected bias reduction.
\end{rem}

Since we can only control the NJODE approximation of the conditional expectation at potential observation times, we need to make an assumption on the training set such that it provides potentially arbitrarily small steps between observation times. Only then can we prove the convergence as $\Delta \to 0$. In practice, this is not necessary, since one ultimately selects some step size to use throughout the approach, by which the limit case is approximated.

\begin{assumption}\label{assumption:6}
    We have $t_{0} = 0$ and assume that there exists a decreasing sequence $D = (\delta_1, \delta_2, \dotsc) \in \R_{>0}^\N$ such that  $\lim_{i\to\infty}\delta_i = 0$ and  $\min_{k\in\N}\P(t_k = t_{k-1} + \delta_i | \mathcal{A}_{t_{k-1}}) = p_i > 0 $.
\end{assumption}
Clearly, $\sum_i p_i \leq 1$ has to hold.
The following example illustrates a setting satisfying this assumption and is the prime example we have in mind.
\begin{example}
    Let $t_0=0$ and $\delta_i = \frac{1}{i}$ for $i\in\N_{>0}$ and for every $k \in \N$ let $\P(t_k = t_{k-1} + \frac{1}{i} | \mathcal{A}_{t_{k-1}}) = \frac{1}{i^2} \frac{6}{\pi^2} = p_i$. Then the sequence of observation times is increasing, $D$ is decreasing with limit $0$ and $\P(t_k -t_{k-1} \in \Pi | \mathcal{A}_{t_{k-1}}) = 1$, i.e., the probability distribution of the observation times is well defined. 
\end{example}
\begin{rem}
    \Cref{assumption:6} is one possibility to ensure convergence as $\Delta \to 0$. 
    A different approach would be to assume that the steps of the observation times have positive density on an interval $(0,\Delta_{\max})$ for some $\Delta_{\max} >0$. However, this changes the following results slightly and makes the argumentation a bit more involved.
\end{rem}

We are now ready to show that the coefficient estimators converge to the true coefficients $\mu,\Sigma$ as the step size $\Delta$ goes to $0$. For this result, we assume that we find the true minimizers of the respective loss functions. In particular, we do not focus on the task of finding the minimizer for the loss function, which is an independent and well-studied problem on its own. Different optimization schemes for global or local convergence exist, which can be combined with our results, as discussed further in \citet[Appendix~E.2]{herrera2021neural}. 
Moreover, $\epsilon$-optimal minimizers yield close approximations as discussed in \citet{NJODE3}.

\begin{theorem}\label{thm:1}
    Let $\theta_{m,N}^{\min} \in \Theta^{\min}_{m,N}$ for every $m,N$ and assume that the current time $t \in [0,T)$ is an observation time, i.e., there exists $k \leq \bar n$ such that $\P(t=t_{k-1} | \mathcal{A}_t)=1$. 
    If \Cref{ass:1,assumption:2,assumption:3,assumption:4,assumption:6} are satisfied, then there exists a sequence $(m_i)_{i \in \N} \in \N^\N$ and a random sequence $(N_i)_{i \in\N}$ taking values in $\N^\N$ such that 
    \begin{equation*}
        \lim_{i \to \infty} \E\left[\lvert \hat \mu_t^{\delta_i,\theta^{\min}_{m_i,N_i}} - \E[\mu_t|\mathcal{A}_t] \rvert_2 \right] 
        = 0 
        = \lim_{i \to \infty} \E\left[\lvert \hat \Sigma_t^{\delta_i,\theta^{\min}_{m_i,N_i}} - \E[ \Sigma_t |\mathcal{A}_t] \rvert_2 \right].
    \end{equation*}
\end{theorem}

The theorem is applicable whenever we are at an observation time, which is enough for our sampling approach, since we always generate the next step from the current observation time and then move to this newly generated observation time.
The sequence of $(N_i)_i$ must be random variables because they depend on the random training set; therefore, one cannot achieve a stronger statement. Vice versa, for any fixed realization of the training set, which is the case in practice, the realization of the sequence $(N_i)_i$ is also fixed.

The estimators converge to the true coefficients instead of their optimal approximations, if they are measurable with respect to the current information. This trivial corollary is stated below.

\begin{cor}\label{cor:1}
    Under the same setting as in \Cref{thm:1}, if additionally we have $\mu_t, \Sigma_t \in \mathcal{A}_t$, then there exists a sequence $(m_i)_{i \in \N} \in \N^\N$ and a random sequence $(N_i)_{i \in\N}$ taking values in $\N^\N$ such that 
    \begin{equation*}
        \lim_{i \to \infty} \E\left[\left\lvert \hat \mu_t^{\delta_i,\theta^{\min}_{m_i,N_i}} - \mu_t \right\rvert_2 \right] 
        = 0 
        = \lim_{i \to \infty} \E\left[\left\lvert \hat \Sigma_t^{\delta_i,\theta^{\min}_{m_i,N_i}} - \Sigma_t  \right\rvert_2\right].
    \end{equation*}
\end{cor}

We split the proof of the theorem into the following lemmas, where we first show the convergence of the idealized estimators to the true coefficients and then the convergence of the realizable estimators to the idealized ones.

\begin{lem}\label{lem:thm1-lem1}
    If \Cref{ass:1} is satisfied we have for each $t\in [0, T]$ that $\P$-a.s.
    \begin{equation}
        \lim_{\Delta \to 0} \E\left[\left\lvert \hat \mu_t^{\Delta} - \E[\mu_t|\mathcal{A}_t] \right\rvert_2 \right] 
        = 0 
        = \lim_{\Delta \to 0} \E\left[\left\lvert \hat \Sigma_t^{\Delta} - \E[\Sigma_t|\mathcal{A}_t]   \right\rvert_2\right].
    \end{equation}
\end{lem}
\begin{proof}
Fix $t\in [0, T]$ and consider the increment
\begin{align}\label{eqn:proof.convergence.0}
	X_{t+\Delta} - X_t = \int_t^{t+\Delta} \mu_s(X_{\cdot \wedge s})\dd s + \int_t^{t+\Delta} \sigma_s(X_{\cdot \wedge s})\dd W_s.
\end{align}
We write $
(X_{t+\Delta} - X_t)^2 = A^2 + 2AM + M^2,
$
with 
\[
A := \int_t^{t+\Delta} b_s(X_{\cdot \wedge s})\dd s, \quad M := \int_t^{t+\Delta} \sigma_s(X_{\cdot \wedge s})\dd W_s\,,
\]
so
\begin{align}\label{eqn:proof.convergence.1}
\mathbb{E} [ (X_{t+\Delta} - X_t)^2 \,|\,\mathcal{A}_t  ] = \mathbb{E}[A^2 \,|\,\mathcal{A}_t] + 2\mathbb{E}[AM \,|\,\mathcal{A}_t] + \mathbb{E}[M^2 \,|\, \mathcal{A}_t]\,.
\end{align}
Since $\mu$ is bounded, we get from \begin{align*}
	\mathbb E[A^2\, |\, \mathcal A_t] &= \mathbb E\bigg[ \bigg( \int_t^{t+\Delta} \mu_s(X_{\cdot \wedge s})\dd s \bigg)^2\, \bigg |\, \mathcal A_t\bigg]
	\end{align*}
	the $\mathbb P$-a.s.~bound $\mathbb E[A^2\, |\, \mathcal A_t]\leq \|\mu\|_\infty^2\Delta^2$, so   
\begin{align}\label{eqn:proof.convergence.2}
	\lim_{\Delta\to 0} \frac 1 \Delta \mathbb{E}[A^2\, |\, \mathcal A_t] = 0\,.
\end{align}

Next note that since $\sigma$ is bounded, It\^o's isometry gives us 
\begin{align} \label{equ:M2} 
	\mathbb{E}[M^2 \,|\,\mathcal{A}_t] &= \mathbb E\bigg [ \bigg| \int_t^{t+\Delta} \sigma_s(X_{\cdot \wedge s})\dd W_s\bigg|^2\,\bigg |\,  \mathcal {A}_t \bigg] 
	  = \mathbb{E} \bigg[ \int_t^{t+\Delta} \sigma_s^2(X_{\cdot \wedge s})\dd s \,\bigg |\, \mathcal{A}_t \bigg]
	\,.
\end{align}
Since $\sigma$ is bounded, dominated convergence, the fundamental theorem of calculus and the continuity of $\sigma$ and the paths of $X$ give $\mathbb P$-a.s.~that
\begin{align}\label {eqn:proof.convergence.3}
	\lim_{\Delta\to 0} \frac 1\Delta \mathbb E[M^2\, |\, \mathcal A_t] &=  \mathbb E\bigg[ \lim_{\Delta\to 0 } \frac{1}{\Delta} \int_t^{t+\Delta}|\sigma_s(X_{\cdot \wedge s})|^2\dd s\, \bigg |\, \mathcal A_t\bigg ]
	  = \mathbb E[\sigma_t^2(X_{\cdot \wedge t})\,|\, \mathcal A_t]\,.
\end{align}
Finally, H\"older's inequality gives 
$
| \mathbb{E}[AM \,|\, \mathcal{A}_t] | \leq ( \mathbb{E}[A^2 \,|\, \mathcal{A}_t] )^{1/2} \,( \mathbb{E}[M^2 \,|\,\mathcal{A}_t] )^{1/2}
$. Since $\mathbb E[M^2\,|\, \mathcal A_t]$ is by It\^o's isometry bounded, we get from \eqref{eqn:proof.convergence.2} that 
\begin{align}\label{eqn:proof.convergence.4} \lim_{\Delta\to 0}\frac 1\Delta |\mathbb E[AM\, |\, \mathcal A_t]| =0\,.\end{align}
Combining via \eqref{eqn:proof.convergence.1} what we found in \eqref{eqn:proof.convergence.2}, \eqref{eqn:proof.convergence.3} and \eqref{eqn:proof.convergence.4} now shows $\mathbb{P}$-a.s.~that 
\begin{align*}
	\lim_{\Delta \to 0} \frac{1}{\Delta} \mathbb{E}[(X_{t+\Delta} - X_t)^2 \,|\, \mathcal{A}_t]  =\mathbb E[\sigma_t^2(X_{\cdot \wedge t})\,|\, \mathcal A_t]\,.
\end{align*}
Therefore, another application of dominated convergence shows that 
\begin{align*}
	\lim_{\Delta\to 0}\mathbb{E} \bigg[\bigg|\frac 1\Delta  \mathbb{E}[(X_{t+\Delta} - X_t)^2 \,|\, \mathcal{A}_t]- E[\sigma_t^2(X_{\cdot \wedge t})\,|\, \mathcal A_t] \bigg| \bigg]=0\,.
\end{align*}

Since $\sigma$ is bounded, so that $M$ in  \eqref{eqn:proof.convergence.0} is a martingale increment, we have that $\mathbb E[M\, |\, \mathcal A_t]=0$, so
\begin{equation}\label{equ:hat mu Delta}
\hat \mu_t^\Delta \coloneqq \frac 1 \Delta \mathbb E[X_{t+\Delta}-X_t\, |\, \mathcal A_t]=\frac 1\Delta \mathbb E\bigg[\int_t^{t+\Delta} \mu_s(X_{\cdot \wedge s}) \dd s\,\bigg |\, \mathcal A_t\bigg]\,.
\end{equation}
Since $\mu$ and the paths of $X$ are continuous, this in turn gives with the fundamental theorem of calculus and  $\hat \mu_t \coloneqq \lim_{\Delta\to 0} \mu_t^\Delta$ for all $t\in [0, T]$ that $ |\E[ \mu_t(X_{\cdot \wedge t}) | \mathcal{A}_t] - \hat \mu_t| = 0$. Now dominated convergence shows for all $t\in [0, T]$ that 
\[\lim_{\Delta\to 0} \mathbb E[|\E[ \mu_t(X_{\cdot \wedge t}) | \mathcal{A}_t] - \hat \mu_t^\Delta|]= \mathbb E[|\E[ \mu_t(X_{\cdot \wedge t}) | \mathcal{A}_t] - \hat \mu_t|]=0 \] 
concluding the proof.
\end{proof}

\begin{lem}\label{lem:thm1-lem2}
    Under the same setting as in \Cref{thm:1}, for any $\epsilon > 0$ and any $i \in \N$ with $t + \delta_i \leq T$, there exists an $m \in \N$ and a random variable $N$ with values in $\N$ such that 
    \begin{equation*}
        \E\left[\left\lvert \hat \mu_t^{\delta_i,\theta^{\min}_{m,N}} - \hat\mu_t^{\delta_i} \right\rvert_2 \right] < \epsilon 
        \quad \text{and} 
        \quad \E\left[\left\lvert \hat \Sigma_t^{\delta_i,\theta^{\min}_{m,N}} - \hat\Sigma_t^{\delta_i} \right\rvert_2 \right] < \epsilon.
    \end{equation*}
\end{lem}
\begin{proof}
First note that it is enough to show the statement for $\mu$, since it follows equivalently for $\Sigma$. Taking the maximum of the values $m,N$ derived for $\mu$ and $\Sigma$, respectively, yields the joint statement. 

We will use results from \citet[Section~5.2]{NJODE3} to rewrite the pseudo-metric $d_k$ under our assumptions on the observation times.
According to our assumptions, we have $k$ such that a.s.\ $t=t_{k-1} \leq T-\delta_i$. 
\Cref{prop:NJODE assumptions satisfied} implies that all assumptions are satisfied to apply the main convergence theorems \citet[Theorems~4.1 and~4.4]{krach2025operator}, showing convergence in the metrics $d_k$ of the output processes of the NJODE models, $G^{\theta^{\min}_{m,N}}$ and $S^{\theta^{\min}_{m,N}}$, to the conditional expectation processes of $X$ and $Z$, respectively.
Using the definition of $\E_k[\cdot] = \E[\cdot \1_{\{n\geq k\}}]/\P(n\geq k)$, \citet[Proposition~5.2, adapted for the extended definition of $d_k$]{NJODE3} then implies that for any $\tilde \epsilon >0$ there exists $m$ and a random variable $N$ such that 
\begin{multline}\label{eq:thm1-lem-2,1}
    \tilde\epsilon > d_k(\hat X, G^{\theta^{\min}_{m,N}}) \\
    \geq \P(n\geq k)^{-1} \, \E \left[ \1_{\{n\geq k\}}
            \sum_{j \in \N} \left( 
                \left\lvert \E[X_{t+\delta_j} | \mathcal{A}_t] - G^{\theta^{\min}_{m,N}}_{t,\delta_j} \right\rvert_2 
                + \left\lvert \E[X_{t+\delta_j} | \mathcal{A}_{t_k}] - G^{\theta^{\min}_{m,N}}_{t+\delta_j,0} \right\rvert_2
            \right) p_j
      \right] \\
      \geq p_i \,  \E \left[ \left\lvert \E[X_{t+\delta_i} | \mathcal{A}_t] - G^{\theta^{\min}_{m,N}}_{t,\delta_i} \right\rvert_2 \right],
\end{multline}
where we used that in the case where $t_k = t+ \delta_i \leq T$ we have   $\1_{\{n\geq k\}}=1$ a.s.\ and  $\P(n \geq k) \leq 1$.
Hence, we have by \eqref{eq:thm1-lem-2,1} that
\begin{equation*}
    \E \left[ \left\lvert \E[X_{t+\delta_i} | \mathcal{A}_t] - G^{\theta^{\min}_{m,N}}_{t,\delta_i} \right\rvert_2 \right] \leq \tilde \epsilon / p_i,
\end{equation*}
and similarly by considering the second term of $d_{k-1}$
\begin{equation*}
    \E \left[ \left\lvert \E[X_{t} | \mathcal{A}_t] - G^{\theta^{\min}_{m,N}}_{t,0} \right\rvert_2 \right] \leq \tilde \epsilon,
\end{equation*}
where we used that under our assumptions $t_{k-1}=t$ a.s.\ (meaning that the random observation time can be replaced by $t$ in the expectation).
With these two bounds we have
\begin{multline*}
     \E\left[\left\lvert \hat \mu_t^{\delta_i,\theta^{\min}_{m,N}} - \hat\mu_t^{\delta_i} \right\rvert_2 \right]
     = \frac{1}{\delta_i} \E\left[\left\lvert E[X_{t + \delta_i} - X_t | \mathcal{A}_t ] - (G_{t,\delta_i}^{\theta^{\min}_{m,N}} - G_{t,0}^{\theta^{\min}_{m,N}}) \right\rvert_2 \right] \\
     \leq \frac{1}{\delta_i} \E\left[\left\lvert E[X_{t + \delta_i} | \mathcal{A}_t ] - G_{t,\delta_i}^{\theta^{\min}_{m,N}} \right\rvert_2 \right] + \frac{1}{\delta_i} \E\left[\left\lvert E[ X_t | \mathcal{A}_t ] -  G_{t,0}^{\theta^{\min}_{m,N}} \right\rvert_2 \right]
     \leq \frac{\tilde \epsilon}{\delta_i p_i} + \frac{\tilde \epsilon}{\delta_i} \leq \frac{2 \tilde \epsilon}{\delta_i p_i},
\end{multline*}
using triangle inequality and that $p_i \leq 1$.
Choosing $\tilde \epsilon \leq \frac{\epsilon \delta_i p_i}{2}$ completes the proof.
\end{proof}

\begin{rem}
    We note that if we are in the case of complete observations or if \Cref{assumption:2} is slightly stronger such that $\min_k \P(M_{k,i}=1) > 0$, then the proven convergence in \Cref{lem:thm1-lem2} is independent of $k$. Indeed, under this assumption, the metric $d_k$ can be bounded in \citet[?]{krach2025operator} by terms not dependent on $k$. Hence, the sequences $(m_i)_i$, $(N_i)_i$ do not depend on $k$, which implies that we converge uniformly at all observation (or sampling) times.
\end{rem}

\begin{proof}[Proof of \Cref{thm:1}.]
Again, we only show the statement for $\mu$, since it follows equivalently for $\Sigma$. 
Let $m_i,N_i$ be chosen such that the statement of \Cref{lem:thm1-lem2} holds for $i$ with $\epsilon_i = 1/i$. Then 
\begin{multline*}
    \lim_{i \to \infty} \E\left[\lvert \hat \mu_t^{\delta_i,\theta^{\min}_{m_i,N_i}} - \E[\mu_t|\mathcal{A}_t] \rvert_2 \right] 
    \leq \lim_{i \to \infty} \left(  
        \E\left[\lvert \hat \mu_t^{\delta_i,\theta^{\min}_{m_i,N_i}} - \hat\mu_t^{\delta_i} \rvert_2 \right] +
        \E\left[\lvert \hat\mu_t^{\delta_i} - \E[\mu_t|\mathcal{A}_t] \rvert_2 \right]
    \right) \\
    \leq \lim_{i \to \infty} \left(  
        \frac{1}{i} +
        \E\left[\lvert \hat\mu_t^{\delta_i} - \E[\mu_t|\mathcal{A}_t] \rvert_2 \right] 
    \right)
    = 0,
\end{multline*}
by triangle inequality and \Cref{lem:thm1-lem1}, since $\lim_{i \to \infty} \delta_i = 0$.
\end{proof}

\begin{cor}
    The statement of \Cref{thm:1} holds equivalently, when using a joint model and joint training for $G^\theta_1,S^\theta_2$ with or without the self-injected bias correction of \Cref{rem:self-injected bias correction}.
\end{cor}
The proof of this corollary follows by adapting \Cref{lem:thm1-lem1,lem:thm1-lem2} accordingly.

\section{Estimating the Instantaneous Coefficients}\label{sec:Estimating the Instantaneous Coefficients}
In \Cref{sec:Details Assumptions and Theoretical Guarantees for the Coefficient Estimates}, we used the quotient of the increment and its square with some step size $\Delta$ to define the idealized estimators, which could naturally be realized through the NJODE's approximation of the respective conditional expectations. Although these estimators are very natural and practical, they depend on the step size $\Delta$. To be precise, they are the average of the conditional expectation of the respective coefficients over the time increment $\Delta$; see \Cref{equ:M2,equ:hat mu Delta}. For fixed $\Delta$ the volatility estimator may thus contain an additional bias term, which vanishes only in the limit. 
Since in practice, we cannot reach the limit we should aim to remedy this undesirable feature. We thus  develop a more sophisticated method to directly estimate instantaneous coefficients. This method improves the quality of the estimator by debiasing it.
In the following, we show how we can tweak our NJODE to obtain estimators of the instantaneous coefficients.

\subsection{Instantaneous Drift estimator}
We first discuss the drift estimator, for which \eqref{equ:mu hat Delta idealized} and \Cref{lem:thm1-lem1} imply that
\begin{equation}\label{equ:right limit of hat mu idealized}
    \lim_{\Delta \downarrow 0} \hat\mu^\Delta = \lim_{\Delta \downarrow 0}  E\left[\frac{X_{t + \Delta} - X_t }{\Delta} \middle| \mathcal{A}_t \right]  \overset{L^1}{=} \E[\mu_t | \mathcal{A}_t].
\end{equation}
Therefore, instead of using the NJODE $G^\theta$ as in \Cref{sec:The Neural Jump ODE Model for Coefficient Estimation} to learn the conditional expectation of $X$, we can use it to learn the conditional expectation of the \emph{increment's quotient} of $X$, i.e., of
\begin{equation}\label{equ:increment quotient of X}
    X^{\text{IQ}}_t = \frac{X_t - X_{\tau(t)}}{t - \tau(t)}.
\end{equation}
Intuitively, if we use $V=X^{\text{IQ}}_t$ as target process for the NJODE $G^\theta$ (with input process $U=X$), then $G^\theta_{\tau(t), t-\tau(t)} \approx \E[X^{\text{IQ}}_t | \mathcal{A}_{\tau(t)}]$ for any $t > \tau(t)$. At observation times $t=\tau(t)$, the target process $X^{\text{IQ}}_t$ is not defined a priori, hence, we do not have a target value to train the NJODE's output after the jump. However, we know from \eqref{equ:right limit of hat mu idealized} that the right-limit of $\E[X^{\text{IQ}}_t | \mathcal{A}_{\tau(t)}]$ for $t \searrow \tau(t)$ is the (conditional expectation of the) instantaneous coefficient $\mu_{\tau(t)}$. 
Therefore, training the NJODE $G^\theta$ with the \emph{noise-adapted} loss function
\begin{equation}
    \Psi_{\text{noisy}}(V, \eta) := \E\left[ \frac{1}{n} \sum_{i=1}^n  \left\lvert \proj{V} (M_i) \odot (V_{t_i-} - \eta_{t_{i}-} ) \right\rvert_2^2  \right], \label{equ:Psi noisy} 
\end{equation}
implies that the model learns to jump to the right-limit $\E[\mu_{\tau(t)} | \mathcal{A}_{\tau(t)}]$ at observation times \citep[see also][Section~3]{NJODE3}. Indeed, since the NJODE prediction evolves continuously after an observation, it would otherwise be different from the optimal prediction right after the observation time, hence, it would not optimize the loss \eqref{equ:Psi noisy} under \Cref{assumption:6}. Therefore, we obtain a direct estimator of the instantaneous drift coefficient $G^\theta_{\tau(t), 0} \approx \E[\mu_{\tau(t)} | \mathcal{A}_{\tau(t)}]$.
In the following theorem, we formalize this result. 
To use dominated convergence, we need to assume that the NJODE output is bounded by some constant. This constant can be chosen as the estimator truncation level $K$ in \Cref{sec:The Generative Procedure}, making this result consistent with the generative procedure of \Cref{sec:The Generative Procedure}. Additionally, we make the technical assumption that the used neural ODEs $f_\theta$ are bounded, such that we can ensure convergence of the model output (see also \Cref{rem:boundedness time derivative}).
\begin{theorem}\label{thm:convergence of instantaneous drift estimator}
    Let $\hat \mu_t^{\theta} = G^\theta_{t,0}$, for the NJODE output $G^\theta$ that is trained with the noise-adapted loss function to predict $V= X^{\text{IQ}}$ from the input $U=X$.
    Let $\theta_{m,N}^{\min} \in \Theta^{\min}_{m,N}$ for every $m,N$ and assume that the current time $t \in [0,T)$ is an observation time, i.e., there exists $k \leq \bar n$ such that $\P(t=t_{k-1} | \mathcal{A}_t)=1$. 
    We assume that $G^\theta$ is bounded by some constant $K$. Moreover, we assume that $\sup_{m,N}|f_{\theta_{m,N}^{\min}}| < K$ and that the assumptions to apply the NJODE convergence results \citep[Theorems~4.1 and~4.4]{krach2025operator} are satisfied by $X^{\text{IQ}}$. 
    If \Cref{ass:1,assumption:2,assumption:3,assumption:4,assumption:6} are satisfied, then there exists a sequence $(m_i)_{i \in \N} \in \N^\N$ and a random sequence $(N_i)_{i \in\N}$ taking values in $\N^\N$ such that 
    \begin{equation*}
        \lim_{i \to \infty} \E\left[\lvert \hat \mu_t^{\theta^{\min}_{m_i,N_i}} - \E[\mu_t|\mathcal{A}_t] \rvert_2 \right] 
        = 0.
    \end{equation*}
\end{theorem}
\begin{proof}
    Since $t$ is an observation time, we have $\tau(t) = t$.
    We use triangle inequality to write for a sequence of parameters $\theta_i$ (that will be chosen later) and for $\delta_i $ as in \Cref{assumption:6},    
    \begin{multline*}
        \lim_{i \to \infty} \E\left[\lvert \hat \mu_t^{\theta_i} - \E[\mu_t|\mathcal{A}_t] \rvert_2 \right] \\
        = \lim_{i \to \infty} \left\{ \E\left[\lvert G^{\theta_i}_{t,0} - G^{\theta_i}_{t,\delta_i}  \rvert_2 \right] 
        + \E\left[\lvert G^{\theta_i}_{t,\delta_i} - \E[X_{t+\delta_i}^{\text{IQ}} | \mathcal{A}_t] \rvert_2 \right] 
        + \E\left[\lvert \E[X_{t+\delta_i}^{\text{IQ}} | \mathcal{A}_t] - \E[\mu_t|\mathcal{A}_t] \rvert_2 \right]  \right\}.
    \end{multline*}
    We show that each of the three terms converges.
    The third term converges to $0$ by \Cref{lem:thm1-lem1}. 
    The middle term converges to $0$ by the NJODE convergence result, similarly as in the proof of \Cref{lem:thm1-lem2}. In particular, \Cref{prop:NJODE assumptions satisfied} and the assumption on $X^{\text{IQ}}$ imply that we can use the main convergence theorems \citet[Theorems~4.1 and~4.4]{krach2025operator}, showing convergence of the NJODE output to the conditional expectation in the metrics $d_k$. As in \Cref{lem:thm1-lem2}, for any $\tilde \epsilon_i > 0$ we find $m_i$ and a random variable $N_i$ such that 
    \begin{equation*}
        \E \left[ \left\lvert \E[X^{\text{IQ}}_{t+\delta_i} | \mathcal{A}_t] - G^{\theta^{\min}_{m_i,N_i}}_{t,\delta_i} \right\rvert_2 \right] \leq \tilde \epsilon_i / p_i.
    \end{equation*}
    Choosing $\tilde \epsilon_i = p_i/i$ and setting $\theta_i \coloneqq \theta^{\min}_{m_i,N_i}$, we therefore have
    \begin{equation*}
        \lim_{i \to \infty} \E\left[\lvert G^{\theta_i}_{t,\delta_i} - \E[X_{t+\delta_i}^{\text{IQ}} | \mathcal{A}_t] \rvert_2 \right] \leq \lim_{i \to \infty} 1/i = 0.
    \end{equation*}
    For the first term we want to use dominated convergence. The integrability of a dominating random variable is implied by the boundedness of $G^\theta$. 
    Moreover, the right-continuous definition \eqref{equ:PD-NJ-ODE} implies that for any fixed $\theta$ we have
    \begin{equation*}
        G^{\theta}_{t,\varepsilon} = g_{\theta}\left(H_t+\int_t^{t+\varepsilon} f_\theta(H_{s-}, s, t,U_t) \dd s\right) \xrightarrow{\varepsilon \downarrow 0} g_{\theta}(H_t) = G^{\theta_i}_{t,0},
    \end{equation*}
    due to the continuity of $g_\theta$ and the fundamental theorem of calculus. However, we need the stronger statement that $G^{\theta_i}_{t,\delta_i} \xrightarrow{i \to \infty} G^{\theta_i}_{t,0}$; in particular, $\theta_i$ changes together with $\varepsilon=\delta_i$. Since $g_\theta$ can be chosen $1$-Lipschitz continuous \citep[see][Proof of Theorem~4.1]{krach2025operator} and since $\sup_i|f_{\theta_i}|$ is bounded by assumption, this stronger convergence holds. Therefore, the first term converges to $0$ by dominated convergence.
\end{proof}
\begin{rem}\label{rem:boundedness time derivative}
    If the time-derivative of the conditional expectation process of $X^{\text{IQ}}$, i.e., the function $f^{X^{\text{IQ}}}$, is bounded on $[0,T]$, then we can choose the neural ODE networks as bounded output NNs (with the bound implied by $f^{X^{\text{IQ}}}$, which they approximate), such that the assumption $\sup_{m,N}|f_{\theta_{m,N}^{\min}}| < K$ is satisfied for some constant $K$. Moreover, weaker assumptions (that do not require the boundedness of $|f_{\theta_{m,N}^{\min}}|$) could be formulated to ensure that $G^{\theta_i}_{t,\delta_i} \xrightarrow{i \to \infty} G^{\theta_i}_{t,0}$ holds, which essentially amounts to a uniform convergence property on $\Theta^{\min}_{m,N}$.
\end{rem}

\subsection{Instantaneous Diffusion estimator}
For the diffusion estimator, we use a similar approach as for the drift estimator. In particular, we define the \emph{quadratic increment's quotient} of $X$, i.e., the quotient of $Z$, as 
\begin{equation}\label{equ:quotient of Z}
    Z^{\text{Q}}_t = \frac{(X_t - X_{\tau(t)}) (X_t - X_{\tau(t)})^\top}{t - \tau(t)}
\end{equation}
and train the NJODE model $S^\theta$ with the noise-adapted loss function \eqref{equ:Psi noisy} to directly predict the target $V=Z^{\text{Q}}$ from the input process $U=X$. Therefore, the same arguments as for the drift estimator imply that the NJODE $S^\theta$ jumps to the right-limit
\begin{equation}\label{equ:right limit of hat diffusion idealized}
    \lim_{\Delta \downarrow 0} \hat\Sigma^\Delta = \lim_{\Delta \downarrow 0}  E\left[\frac{(X_{t + \Delta} - X_t)(X_{t + \Delta} - X_t)^\top }{\Delta} \middle| \mathcal{A}_t \right] =  \lim_{\Delta \downarrow 0}  E\left[ Z^\text{Q}_{t+\Delta} \middle| \mathcal{A}_t \right] \overset{L^1}{=} \E[\Sigma_t | \mathcal{A}_t].
\end{equation}
at observation times $t=\tau(t)$ (see~\eqref{equ:sigma hat Delta idealized} and \Cref{lem:thm1-lem1}). This yields a direct estimator of the instantaneous diffusion coefficient $S^\theta_{\tau(t), 0} \approx \E[\Sigma_{\tau(t)} | \mathcal{A}_{\tau(t)}]$, as is formalized in the following theorem.
\begin{theorem}\label{thm:convergence of instantaneous diffusion estimator}
    Let $\hat \Sigma_t^{\theta} = S^\theta_{t,0}$, for the NJODE output $S^\theta$ that is trained with the noise-adapted loss function to predict $V=Z^{\text{Q}}$ from the input $U=X$.
    Let $\theta_{m,N}^{\min} \in \Theta^{\min}_{m,N}$ for every $m,N$ and assume that the current time $t \in [0,T)$ is an observation time, i.e., there exists $k \leq \bar n$ such that $\P(t=t_{k-1} | \mathcal{A}_t)=1$. 
    We assume that $S^\theta$ is bounded by some constant $K$. Moreover, we assume that $\sup_{m,N}|f_{\theta_{m,N}^{\min}}| < K$ and that the assumptions to apply the NJODE convergence results \citep[Theorems~4.1 and~4.4]{krach2025operator} are satisfied by $Z^{\text{Q}}$. 
    If \Cref{ass:1,assumption:2,assumption:3,assumption:4,assumption:6} are satisfied, then there exists a sequence $(m_i)_{i \in \N} \in \N^\N$ and a random sequence $(N_i)_{i \in\N}$ taking values in $\N^\N$ such that 
    \begin{equation*}
        \lim_{i \to \infty} \E\left[\lvert \hat \Sigma_t^{\theta^{\min}_{m_i,N_i}} - \E[\Sigma_t|\mathcal{A}_t] \rvert_2 \right] 
        = 0.
    \end{equation*}
\end{theorem}
The proof of this theorem follows by adapting the proof of \Cref{thm:convergence of instantaneous drift estimator} accordingly. Moreover, \Cref{rem:boundedness time derivative} applies equivalently.

\begin{rem}\label{rem:joint instantaneous coeff estimation}
    The instantaneous estimators also have a computational advantage over the baseline estimators of \Cref{sec:Details Assumptions and Theoretical Guarantees for the Coefficient Estimates}. In particular, including the division by the step size $\Delta = t - \tau(t)$ in the definition of the target processes $X^{\text{IQ}}, Z^\text{Q}$, the range of values of these target processes becomes smaller, in general. Therefore, the NJODE should better approximate them, reducing the absolute error of the model, due to similar arguments as in \Cref{rem:self-injected bias correction}.  
    However, the target process $Z^\text{Q}$ for the diffusion estimator still includes a bias term, which can be reduced similarly as in \Cref{rem:self-injected bias correction}. In particular, we can consider the process of the \emph{quadratic bias-corrected increment's quotient} of $X$, i.e., the quotient of $Z^\text{BC}$,
    \begin{equation}\label{equ:BC quotient of Z}
        Z^{\text{BCQ}}_t = \frac{(X_t - \E[X_{t}|\mathcal{A}_{\tau(t)}]) (X_t - \E[X_{t}|\mathcal{A}_{\tau(t)}])^\top}{t - \tau(t)},
    \end{equation}
    which decreases the value of its conditional expectation $\E[Z^{\text{BCQ}}_t | \mathcal{A}_{\tau(t)}]$. 
    While we do not have access to $\E[X_{t}|\mathcal{A}_{\tau(t)}]$, the NJODE output $G^\theta$ approximates $\E[X^{\text{IQ}}_t | \mathcal{A}_{\tau(t)}]$, which yields the approximation 
    $$(t - \tau(t)) G^\theta_t + X_{\tau(t)} \approx \E[X_{t}|\mathcal{A}_{\tau(t)}].$$ This can be used to define
    \begin{equation*}
        \tilde Z^{\text{BCQ}}_t = \frac{\left(X_t - X_{\tau(t)} -   (t - \tau(t)) G^\theta_t\right) \left(X_t - X_{\tau(t)} -   (t - \tau(t)) G^\theta_t\right)^\top}{t - \tau(t)} = (t - \tau(t))(X^{\text{IQ}}_t - G^\theta_t) (X^{\text{IQ}}_t -  G^\theta_t)^\top,
    \end{equation*}
    as target for training $S^\theta$. By training a joint model for $G^\theta, S^\theta$, this leads to a self-injected bias reduction.
\end{rem}


\section{The Generative Procedure}\label{sec:The Generative Procedure}

For the results in this section, we refine \Cref{ass:1} and impose the following.
\begingroup
\renewcommand{\theassumption}{1'}
\begin{assumption}\label{ass:1.prime}
    The dynamics of the diffusion process $X$ are given by
    \begin{equation}\label{eq:def X prime}
        X_t = x_0 + \int_0^t \mu_s(X_{s})\dd s + \int_0^t \sigma_s(X_{s})\dd W_s,\, \qquad \text{for}~t\in [0, T],
    \end{equation}
    where $\mu$ and $\sigma$ are continuous and bounded functions on $[0, T]\times \R^d$ with values in $\R^d$ and $\R^{d\times m}$, respectively,  and $W=(W_t)_{t\in [0, T]}$ is an $m$-dimensional standard Brownian motion. In addition, we assume that $x\mapsto \sigma_t(x)$ is uniformly H\"older-continuous and that $\sigma_t \sigma^\top_t$ is uniformly elliptic (uniformly positive definite). 
\end{assumption}
\endgroup
Under this assumption it is a classical result that the law of $X$ is unique; see e.g.~\cite[Thm.~3.2.1]{stroock2007multidimensional}.

In this section, we use the learned characteristics $(\hat \mu_t^{\delta_i,\theta^{\min}_{m_i,N_i}})_{i \in \N}$ and $(\hat \Sigma_t^{\delta_i,\theta^{\min}_{m_i,N_i}})_{i \in \N}$ for a generative task. 
In the following procedure and results, these baseline estimators of \Cref{sec:Details Assumptions and Theoretical Guarantees for the Coefficient Estimates}, can equivalently be replaced by the more sophisticated instantaneous estimators $(\hat \mu_t^{\theta^{\min}_{m_i,N_i}})_{i \in \N}$ and $(\hat \Sigma_t^{\theta^{\min}_{m_i,N_i}})_{i \in \N}$ of \Cref{sec:Estimating the Instantaneous Coefficients}. 
Via an Euler-Maruyama scheme, we construct approximate laws $(\hat {\mathbb P}^i)_{i \in \N}$, which we show to converge to the true law of the underlying process $X$, as $i \to \infty$. 

Consider a fixed time $\bar t \in [0, T]$. This can be an observation time, but it does not need to. In case it is not an observation time, we define $\bar t' = \tau(\bar t)$ to be the last observation time before and use $\bar t '$ instead of $\bar t$, to simplify the notation.
At this time we have collected $\bar k \coloneqq \kappa(\bar t)$ observations which give us the history $(O_1, \ldots, O_{\bar k}, 0, \ldots)\in \mathscr O^{\mathbb N}$. We next adapt the observation framework from \Cref{sec:information sigma-algebra} into a ``simulation framework''. For this we fix $\delta>0$ and extend in \eqref{eqn:obs.times} the observation times before $\bar t$\footnote{%
These observations before $\bar t$ are the observed history on which we want to condition.
} with deterministic $\delta$-spaced times after $\bar t$.  With the notation from \Cref{assumption:6}, this leads to $\P(t_{\kappa(\bar t)+k}= \bar t+ \delta k\, |\, \mathcal A_{\bar t})=1$, so that $0 = t_0 < t_1< \cdots < t_{\bar k}= \bar t < t_{\bar k+1}<t_{\bar k+2}<\cdots \leq T $ becomes 
\begin{align*}
	0 = t_0 < t_1< \cdots < t_{\bar k }= \bar t < t_{\bar k}+\delta< t_{\bar k}+2\delta<\cdots \leq T\,,
\end{align*} 
with $M_{\bar t + m \delta, l} = 1$ for all $m \in \N_{\geq 1}$ and $1\leq l\leq d$. With this, $X_{\bar t + m \delta}\odot M_{\bar t + m \delta}$ simply becomes $X_{\bar t + m \delta}$, which is to say that after time $\bar t$ we have full observations (since we generate them ourselves). 

In this setting, we run the following online estimation and simulation scheme. We start with an $d$-dimensional $(\mathbb F, \mathbb P)$-Brownian motion $B=(B_t)_{t \in [0, T]}$, which we take to be independent of the probabilistic framework we presented thus far and fix a number $K > 3 \max\{\|\mu\|_{\infty}, \|\Sigma\|_{\infty} \}$. We choose and fix $\theta \in \Theta$, and compute as in \Cref{sec:Details Assumptions and Theoretical Guarantees for the Coefficient Estimates} at the initial time $\bar t$ the prediction of the present state $ \tilde X_{\bar t}\coloneqq G_{\bar t, 0}$ and evaluate the learned coefficients $\hat \mu_{\bar t}^{\delta,\theta}$ and $\hat \Sigma_{\bar t}^{\delta,\theta}$ which we truncate at $K$ to make them bounded. 
In particular, we define 
\begin{equation*}
    (\hat \mu_{\bar t}^{\delta,\theta})_K \coloneqq (\hat \mu_{\bar t}^{\delta,\theta} \wedge K) \vee -K \quad \text{and} \quad (\hat \Sigma_{\bar t}^{\delta,\theta})_K \coloneqq (\hat \Sigma_{\bar t}^{\delta,\theta} \wedge K) \vee -K,
\end{equation*}
and use the same notation $(\cdot)_K$ also for other coefficients.
We use these to simulate the first step of the Euler-Maruyama scheme as
\begin{align}\label{eqn:generative.euler.1}
	\tilde X_{\bar t+h} = \tilde X_{\bar t} +(\hat \mu_{\bar t}^{\delta, \theta} )_K h + (\hat \Sigma_{\bar t}^{\delta, \theta})_K ^{1/2} (B_{\bar t+h}- B_{\bar t})\qquad \text{for}~ h \in (0, \delta]\,,
\end{align}
where $(\hat \Sigma_{\bar t}^{\delta, \theta})^{1/2}$ is a symmetric positive semi-definite square-root\footnote{%
If the estimator $\hat \Sigma_{\bar t}^{\delta, \theta}$ is strictly positive-definite, then there exists a unique positive-definite square-root. However, in general, the estimator as we defined it can become positive semi-definite, hence, multiple symmetric positive semi-definite square-roots can exist, out of which we choose one.
} of the $d\times d$-matrix $\hat \Sigma_{\bar t}^{\delta, \theta}$. In the notation, we do not make $K$ explicit, but $\tilde X_{\bar t+ h}$ clearly depends on the choice of $K$. In the $(m+1)$'st step, i.e.~starting at time $t= t_{\bar k}+m\delta$, we have collected the observations $(O_{1}, \ldots, O_{\bar k}, O_{\bar k +1}, \ldots, O_{\bar k + m}, 0,  \ldots )$ comprised of the potentially partially observed real-world data before $\bar t$, and of the fully observed generated samples after time $\bar t$. We then compute $\hat \mu_{\bar t+ m \delta}^{\delta, \theta}$ and $\hat \Sigma_{\bar t+m \delta}\in \R^{d\times d}$ as in \Cref{sec:Details Assumptions and Theoretical Guarantees for the Coefficient Estimates}, which we use to simulate  
\begin{align}\label{eqn:generative.euler.2}
	\tilde X_{\bar t+ m \delta+h} = \tilde X_{\bar t+ m \delta} +(\hat \mu_{\bar t+ m\delta }^{\delta, \theta})_K h + (\hat \Sigma_{\bar t+m\delta }^{\delta, \theta})_K^{1/2} (B_{\bar t + m\delta +h}- B_{\bar t+ m \delta})\qquad \text{for}~ h \in (0, \delta]\,.
\end{align}
This concludes the description of the generative sampling scheme. 
So far, the coefficient estimates are defined on the generation grid only; to define them on the entire interval $[0,T]$, we simply use their constant continuations
\begin{equation}\label{equ:constant continuation of estimated coefs}
    \hat \mu_{\bar t+ m\delta + h}^{\delta, \theta} \coloneqq \hat \mu_{\bar t+ m\delta }^{\delta, \theta} \quad \text{and} \quad \hat \Sigma_{\bar t+m\delta + h }^{\delta, \theta} \coloneqq \hat \Sigma_{\bar t+m\delta }^{\delta, \theta} \quad \text{for}~ h \in [0,\delta).
\end{equation}
This definition is consistent in the sense that the solution of the SDE
\begin{equation*}
    \tilde X_t = \tilde X_{\bar t} + \int_{\bar t}^t (\hat \mu_{s}^{\delta, \theta})_K \dd s + \int_{\bar t}^t (\hat \Sigma_{s }^{\delta, \theta})_K \dd B_s \quad \text{for}~ t \in [\bar t, T],
\end{equation*}
coincides with the Euler scheme in \Cref{eqn:generative.euler.1,eqn:generative.euler.2}.

\subsection{Convergence of the Generative Sampling Scheme}\label{sec:Convergence of the Generative Sampling Scheme}
Let $D=(\delta_i)_{i \in \N}$ be a sequence as in \Cref{assumption:6}, and define for each $i\in \N$ a sampling scheme as above using the learned coefficients $(\hat \mu^{\delta_i,\theta^{\min}_{m_i,N_i}})_{i \in \N}$ and $(\hat \Sigma^{\delta_i,\theta^{\min}_{m_i,N_i}})_{i \in \N}$. Let $\tilde X^i=(\tilde X^i)_{t \in [\bar t, T]}$ be the process obtained via \Cref{eqn:generative.euler.1,eqn:generative.euler.2} and define $\P^i\coloneqq \operatorname{Law}_{\mathbb P}(\tilde X^i)$. 

\begin{lemma}
	Let $K>3 \max \{\|\mu\|_{\infty}, \|\Sigma\|_{\infty}\}$. Under Assumptions \ref{ass:1.prime} and \ref{assumption:2}--\ref{assumption:6}, the sequence $(\P^i)_{i \in \mathbb N}$ is tight in $\mathbf C([\bar t, T]; \, \R^d)$.
\end{lemma}

\begin{proof}
	Since the coefficients in \eqref{eqn:generative.euler.1}, \eqref{eqn:generative.euler.2} are bounded, the result follows from standard characterizations for tightness for diffusion processes; see e.g.~\cite[Thm.~1.4.6]{stroock2007multidimensional}.
\end{proof}
To simplify notation we write $\hat \mu_t^i\coloneqq \hat \mu_t^{\delta_i,\theta^{\min}_{m_i,N_i}}$ and $\hat \Sigma_t^i\coloneqq \hat \Sigma_t^{\delta_i,\theta^{\min}_{m_i,N_i}}$ in the sequel. To state the main result of this section, we need an auxiliary lemma.

\begin{lemma}\label{lem:thm.2-lem1}
	Let $K > 3\max \{\|\mu\|_{\infty}, \|\Sigma\|_{\infty} \}$. Under Assumptions \ref{ass:1.prime} and \ref{assumption:2}--\ref{assumption:6}, and in the setting of \Cref{thm:1} and \Cref{cor:1} (in particular, $t$ is an observation time), there exists a sequence $(m_i)_{i \in \N} \in \N^\N$ and a random sequence $(N_i)_{i \in\N}$ taking values in $\N^\N$ such that 
    \begin{equation*}
        \lim_{i \to \infty} \E\left[\left\lvert (\hat \mu_t^{i})_K  - \mu_t \right\rvert_2 \right] 
        = 0 
        = \lim_{i \to \infty} \E\left[\left\lvert (\hat \Sigma_t^{i})_K - \Sigma_t  \right\rvert_2\right].
    \end{equation*}
\end{lemma}
\begin{proof}
	 We deduce the result from \Cref{cor:1}. First observe that we have 
	\begin{align}\label{eqn:thm.2-lem1.proof.1}
		  \E[\lvert (\hat \mu_t^{i})_K - \mu_t \rvert_2 ] 
		  &= \E[(\mathbf 1_{|\hat \mu_t^i| \leq K}+\mathbf 1_{|\hat \mu_t^i| > K})\lvert (\hat \mu_t^{i})_K - \mu_t \rvert_2 ]
		  \leq \E[\lvert \hat \mu_t^{i} - \mu_t \rvert_2 ] + \E[\mathbf 1_{|\hat \mu_t^i| > K}\lvert K - \mu_t \rvert_2 ]. 
	\end{align}
	By \Cref{cor:1}, the first term on the right-hand side of \eqref{eqn:thm.2-lem1.proof.1} vanishes as $i \to \infty$. For the second term note that since by \Cref{ass:1} $\mu$ is bounded and $K > |\mu_t|$ we have that $|K-\mu_t| \leq  c$ for some constant $c< \infty$. Therefore, 
	\begin{align*}
		\E[\mathbf 1_{|\hat \mu_t^i| > K}\lvert K - \mu_t \rvert_2 ] \leq c \E[\mathbf 1_{|\hat \mu_t^i| > K}] = c \, \P [\{|\hat \mu_t^i| > K\}].
	\end{align*}  
	For $\epsilon>0$ sufficiently small we have that $\{|\hat \mu_t^i| > K \}\subseteq A^\epsilon_t \coloneqq \{ |\hat \mu_t^i-\mu_t|_2>\epsilon\}$. In fact, recall that $K>3\|\mu\|_\infty$. Therefore, if $\hat \mu^i_t >K$, then $\hat \mu_t^i - \mu_t^i > K - \|\mu\|_\infty>\epsilon$ and if $-\hat \mu^i_t >K$ then $\mu_t^i - \hat \mu_t^i > K - \|\mu\|_\infty>\epsilon$. Now, by \Cref{cor:1} and Markov's inequality, $\P[A_t^\epsilon] \leq \E[|\mu_t-\hat \mu^i_t\|/\epsilon\to 0$ as $i \to \infty$. Thus also the second term on the right-hand side of \eqref{eqn:thm.2-lem1.proof.1} vanishes as $i \to \infty$. With this we immediately deduce the result for $\mu$. The proof for the case of $\Sigma_t$ is analogous. 
\end{proof}

With \Cref{lem:thm.2-lem1}, \eqref{equ:constant continuation of estimated coefs}, triangle-inequality, continuity of the true coefficients (\Cref{ass:1}) and dominated convergence, we get for the truncated coefficients $((\hat \mu_t^i)_K)_{i \in \N}$ that 
	\begin{align*}
		\lim_{i \to \infty} \int_0^T \E[| (\hat \mu_t^i)_K-\mu_t|_2]\, \dd t \leq \int_0^T \lim_{i \to \infty} \left( \E[| (\hat \mu_{\tau(t)}^i)_K -\mu_{\tau(t)}|_2] + \sup_{h \in [0, \delta_i]}\E[| \mu_{\tau(t)} - \mu_{\tau(t) + h}|_2] \right) \, \dd t =  0\,.
	\end{align*}
	Let $\hat \mu^\infty$ be the $\mathbb L^1([0, T]\times \Omega, \dd t\otimes \dd \mathbb P; \R^d)$-limit of $((\hat \mu_t^i)_K)_{i \in \N}$. Similarly, for $((\hat \Sigma_t^i)_K)_{i \in \N}$ we have that 
	\begin{align*}
		\lim_{i \to \infty} \int_0^T \E[|(\hat \Sigma_t^i)_K-\Sigma_t|_2]\, \dd t \leq \int_0^T \lim_{i \to \infty}\left( \E[|(\hat \Sigma_{\tau(t)}^i)_K-\Sigma_{\tau(t)}|_2] + \sup_{h \in [0, \delta_i]}\E[| \Sigma_{\tau(t)} - \Sigma_{\tau(t) + h}|_2] \right)\, \dd t =  0\,,
	\end{align*}
	and we let $\hat \Sigma^\infty$ denote the $\mathbb L^1([0, T]\times \Omega, \dd t\otimes \dd \mathbb P;\, \R^{d\times m})$-limit of $((\hat \Sigma_t^i)_K)_{i \in \N}$. We next choose functions 
	\begin{align*}
		\mu^\infty :[0, T]\times \R^d \to \R^d \qquad \text{and}\qquad \Sigma^\infty :[0, T]\times \R^d\to \R^{d\times m}
	\end{align*} 
	satisfying $\mu^\infty_t(X_t)=\hat \mu_t^\infty$ and $\Sigma^\infty_t(X_t)=\hat \Sigma_t^\infty$ $\dd t\otimes \dd \mathbb P$-a.e.
\begin{theorem}\label{thm:thm.2}
	Under Assumptions \ref{ass:1.prime} and \ref{assumption:2}--\ref{assumption:6}, let $\hat {\mathbb P}^{\infty}$ be a cluster point of $(\hat{\mathbb P}^i)_{i \in \N}$. Suppose that $\hat {\mathbb P}^\infty$ solves the martingale problem for $(x_0, \mu^\infty, \Sigma^\infty)$, or equivalently, that $\mathbb {\hat P}$ is a weak solution of the SDE
	\begin{align*}
		Y_t = x_0 + \int_0^t \mu_s^\infty(Y_s)\, \dd s + \int_0^t (\Sigma_s^\infty)^{1/2}\, \dd W_s\qquad \text{for}~t \in [0, T]\,.
	\end{align*}
	Then $\hat {\mathbb P}^\infty= \operatorname{Law}_{\mathbb P}((X_{t})_{t \in [\bar t, T]})$.
\end{theorem} 

The requirement that $\hat {\mathbb P}$ solves the martingale problem for $(x_0, \mu^\infty, \Sigma^\infty)$ amounts to a stability property of the sequence of semimartingale laws $(\hat { \mathbb P}^i)_{i \in \N}$ and its cluster points $\hat{\mathbb P}^\infty$. We refer to e.g.~\cite[Ch.~11.3]{stroock2007multidimensional} for a classical treatment, or to  \cite{figalli2008existence} for more recent results.

\begin{proof}[Proof of \Cref{thm:thm.2}]
	Since $\hat {\mathbb P}^\infty$ solves the martingale problem for $(x_0, \mu^\infty, \Sigma^\infty)$, we have for all $f \in \C_c^\infty(\R^d)$ and $0 \leq s\leq t\leq T$ that $\hat {\mathbb P}^\infty$-a.s.
		\begin{align*}
		\E^{\hat{\mathbb{P}}^\infty}\bigg[f(X_t)-f(X_s) - \int_s^t({\mathcal L}^\infty_rf)(X_r)\, \dd r \, \Big|\, \mathcal F_s\bigg]= 	\E^{\hat{\mathbb{P}}^\infty} [f(X_t)]-f(X_s) - \int_s^t\E^{\hat{\mathbb{P}}^\infty}[( {\mathcal L}^\infty_rf)(X_r)\, |\, \mathcal F_s]\, \dd r =0\,,
	\end{align*}
	where 
	\begin{align*}
		(\mathcal {L}^\infty_tf)(x) = \mu_t^\infty(x)\cdot \nabla f(x) + \frac 12 \big(\nabla f(x)\big)^\top \Sigma_t^\infty(x) \nabla f(x)
	\end{align*}
	After an application of Fubini's theorem, we can rewrite the conditional expectations under $\hat {\mathbb P}$ in terms of a kernels $\hat p$, giving 
	\begin{align*}
		\int_{\R^d}f(y)\hat p_{s, t}(X_s, \dd y) -f(X_s) - \int_s^t \int_{\R^d}\big({\mathcal L}^\infty f_r(y)\big)\hat p_{s, r}(X_s, \dd y)\, \dd r = 0\,.
	\end{align*}
	Since $\Sigma^\infty_t$ is by \Cref{ass:1.prime} uniformly elliptic, we get for any $r>s$ that the kernel $p_{s, r}(X_s, \dd y)$ is $\hat{\mathbb P}$-a.s.~absolutely continuous with respect to Lebesgue measure; see \cite[Thm.9.1.1]{stroock2007multidimensional} and  \cite[Ch.~1]{porper1984two}. We therefore get a density $q_{s, r}:\R^d\times \R^d \to \R_{\geq 0}$ with which we write $p_{s, r}(X_s, \dd y)= q_{s, r}(X_s, y)\, \dd y$ so that the display just above becomes
	\begin{align*}
		\int_{\R^d}f(y)q_{s, t}(X_s,  y)\, \dd y-f(X_s) - \int_s^t \int_{\R^d}\big(\mathcal L^\infty f_r(y)\big)p_{s, r}(X_s, y)\, \dd y\, \dd r = 0\,.
	\end{align*}
	Using once more the representation of expectations under $\hat {\mathbb P}$ in terms of $ \hat p$, the existence of a density, and \Cref{lem:thm.2-lem1}, we get that $\E^{\hat{\mathbb P}}[|\mu_t^\infty - \mu_t|_2] = \int _{\R^d} |\mu^\infty_t(x) - \mu_t|_2\, \hat q_{0, t}(x_0, x)\dd x=0$. Integrating over $t \in [0, T]$, we find that  $\mu^\infty = \mu$ $\dd t\otimes \dd x$-a.s. Therefore in the display just above we can replace $\mu^\infty$ by $\mu$ and $\Sigma^\infty$ by $\Sigma$ without changing the value of the integral. This gives
	\begin{align*}
		\int_{\R^d}f(y)q_{s, t}(X_s,  y)\, \dd y-f(X_s) - \int_s^t \int_{\R^d}\big(\mathcal Lf_r(y)\big)p_{s, r}(X_s, y)\, \dd y\, \dd r = 0\,.
	\end{align*}
	where 
	$
		(\mathcal {L}_tf)(x) = \mu_t(x)\cdot \nabla f(x) + \frac 12 (\nabla f(x))^\top \Sigma_t(x) \nabla f(x)
	$. This means that $\hat {\mathbb P}^\infty$ solves the martingale problem for $(x_0, \mu, \Sigma)$, i.e.~$\hat {\mathbb P}^\infty$ is the law of a weak solution of the dynamics in \eqref{eq:def X}. But by the comment immediately after \Cref{ass:1.prime} this law is unique. It follows that we must have $\mathbb P = \mathbb {\hat P}^\infty$. 
\end{proof}




\section{Experiments}\label{sec:Experiments}

In this section, we apply the coefficient estimation and path generation procedure of the previous sections to several datasets and compare the different proposed approaches. We also use examples that do not satisfy all the assumptions about the underlying process (\Cref{ass:1}), showing that empirically our method works in more general settings than those for which we were able to derive theoretical guarantees. For example, theoretical boundedness of the parameters is not essential in practice (and not satisfied, for example, by a geometric Brownian motion considered in \Cref{sec:Geometric Brownian Motion}), where the observed parameters are empirically bounded. Nevertheless, we need this assumption in our proofs. Moreover, our NJODE based method can naturally deal with irregular and incomplete observations in the training set (as well as in the initial sequence to be conditioned on) and it can handle path dependence in the parameters (in contrast to \Cref{ass:1.prime}).

The code for running the experiments is available at \url{https://github.com/FlorianKrach/PD-NJODE} and additional details about the implementation can be found in \Cref{sec:Details for implementation}.

\subsection{Geometric Brownian Motion}\label{sec:Geometric Brownian Motion}
We consider a one-dimensional geometric Brownian motion (GBM) satisfying the SDE
\begin{equation*}
    \dd X_t = \mu X_t \dd t + \sigma X_t \dd W_t,
\end{equation*}
where $\mu, \sigma >0$ are constants and $W$ is a Brownian motion. We use the parameters $\mu = 2, \sigma = 0.3$ and set $X_0 = 1$. In the following, we compare the 4 different coefficient estimation approaches introduced in \Cref{sec:Details Assumptions and Theoretical Guarantees for the Coefficient Estimates,sec:Estimating the Instantaneous Coefficients}: 
\begin{itemize}
    \item the baseline estimators of the drift and diffusion coefficient trained separately (cf.~\Cref{thm:1}) \textbf{(Base)},
    \item the baseline estimators of the drift and diffusion coefficient trained jointly with bias-reduction for the diffusion estimator (cf.~\Cref{rem:self-injected bias correction}) \textbf{(Joint Base)},
    \item the instantaneous drift and diffusion coefficient estimators trained separately (cf.~\Cref{thm:convergence of instantaneous drift estimator,thm:convergence of instantaneous diffusion estimator}) \textbf{(Instant)}, and
    \item the instantaneous drift and diffusion coefficient estimators trained jointly with bias-reduction for the diffusion estimator (cf.~\Cref{rem:joint instantaneous coeff estimation}) \textbf{(Joint Instant)}.
\end{itemize}
For all methods, we use the same training dataset (with the special input and output feature processes added individually by necessity) and comparable training. After training, we use the learned estimators of each approach to generate $5000$ new paths starting from $X_0$. We use a standard estimator \citep[see the financial estimator in][Example~2]{heiss2024nonparametric} to compute the estimated values of $\mu, \sigma$ on each of the sets of the generated paths\footnote{%
This estimator uses the knowledge of the distribution of $X$ to compute $\mu,\sigma$ over the entire paths. 
In contrast to this, our drift and diffusion estimators do not use any distributional knowledge, but only the training paths, to estimate the current values of drift $\mu_t = \mu X_t$ and diffusion $\sigma_t = \sigma X_t$, which is much more difficult. %
}. 
Generated paths with invalid values for a GBM, i.e., values $\leq 0$, are excluded before computing these estimates.
Since the models do not get any information about the true underlying model except for the paths of the training set, we compare these estimates to the corresponding estimates on the paths of the training dataset \textbf{(Reference)}. These estimates constitute the retrievable ground-truth, while the true values $\mu,\sigma$ are concealed. The results of the different methods are shown in \Cref{tab:results GBM comparison}. 
We can see an increase in quality with the increasing complexity of the estimation method. 
For the base method, we see a much too large variance in the generated paths, which results from the inaccuracy in the learning method. In particular, without bias reduction, the prediction for one step of $\Delta=0.01$ ahead contains a small upward bias, leading to an overestimation of $\sigma$. Moreover, an error of size $\epsilon$ in the prediction of $(G^\theta_2)_{t,\Delta} \approx \sqrt{\E[Z_{t+\Delta} | \mathcal{A}_t]}$ leads to an error of $\epsilon/\sqrt{\Delta}= 10\epsilon$ in the estimated diffusion coefficient. The base method is the only method that leads to invalid paths for roughly $4.7\%$ of its generated samples.
For the joint base method, the bias reduction helps to significantly improve the estimates of the diffusion, but the use of instantaneous estimates leads to an even greater improvement. 
As suggested by our theoretical analysis, the joint instantaneous method (including bias-reduction for the diffusion estimate) clearly outperforms all others and leads to a generated dataset with estimated parameters $\mu,\sigma$ very similar to those of the training set. 
In the following, we therefore focus on this method and do not report further results for the other ones.

In \Cref{fig:GBM gen paths} we plot $1000$ training and generated paths each. Visually, the distributions look nearly identical. To further verify this, the distributions of $X_t$ at $t=T/2=0.5$ and $t=T=1$ of the true and generated paths are plotted in \Cref{fig:GBM distribution}, which shows a very good match.
Moreover, in \Cref{fig:GBM coeffs} we show the estimated and true drift and diffusion coefficients along one generated path. We see that the joint instantaneous method replicates the true coefficients with high accuracy. 

As described in \Cref{sec:The Generative Procedure}, we can equivalently use the generative method to generate new samples based on a given history of observations. In \Cref{fig:GBM gen paths starting from sequence} we use the first training path until $t=0.55$ as the starting sequence, after which $1000$ different continuations of the path are generated.

\begin{table}[tb]
\caption{Geometric Brownian motion parameters $\mu,\sigma$ estimated (via standard method) on datasets generated based on differently learnt drift and diffusion coefficient estimators. As reference, we show the estimated parameters on the training dataset, which was used to train the coefficient estimators.}
\label{tab:results GBM comparison}
\begin{center}
\begin{tabular}{l | r r r }
\toprule
    & $\mu$ & $\sigma$ & \# invalid paths \\
\midrule[\heavyrulewidth]
True params & 2.0 & 0.3 & - \\
Reference & 1.9841 & 0.2941 & 0 \\
\midrule
Base & 2.1478 &	0.8154 & 234 \\
Joint Base & 2.0892 & 0.2344 & 0 \\
Instant & 1.8717 &	0.2575 & 0 \\
Joint Instant & 1.9619 & 0.2974 & 0 \\
\midrule
Joint Instant 1-step & 1.9819 & 0.2909 & 0 \\
\bottomrule
\end{tabular}
\end{center}
\end{table}

\begin{figure}
\centering
\includegraphics[width=1\textwidth]{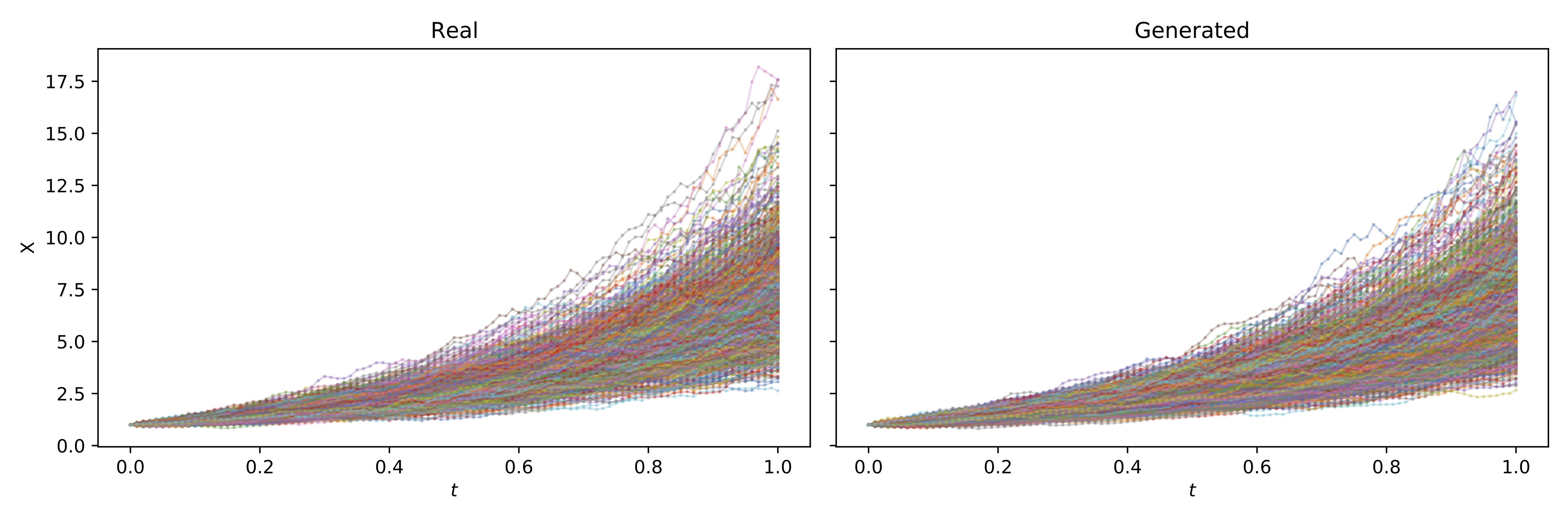}
\caption{Plot of true training paths and generated (with joint instantaneous method) paths, with $1000$ samples each.}
\label{fig:GBM gen paths}
\end{figure}

\begin{figure}
\centering
\includegraphics[width=0.49\textwidth]{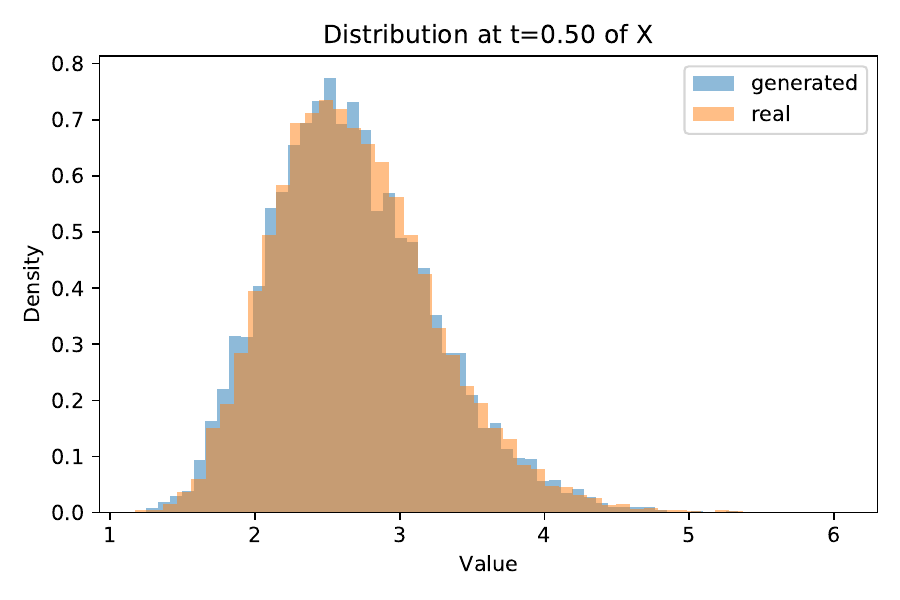}
\includegraphics[width=0.49\textwidth]{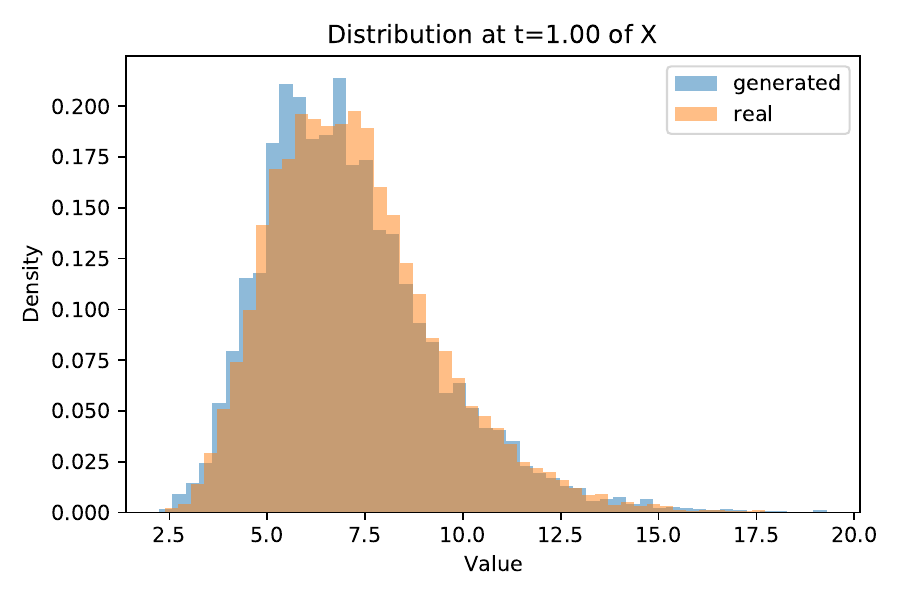}
\caption{Distribution of $X_t$ at $t=0.5$ and $t=T=1$ of true training paths and generated (with joint instantaneous method) paths.}
\label{fig:GBM distribution}
\end{figure}

\begin{figure}
\centering
\includegraphics[width=1\textwidth]{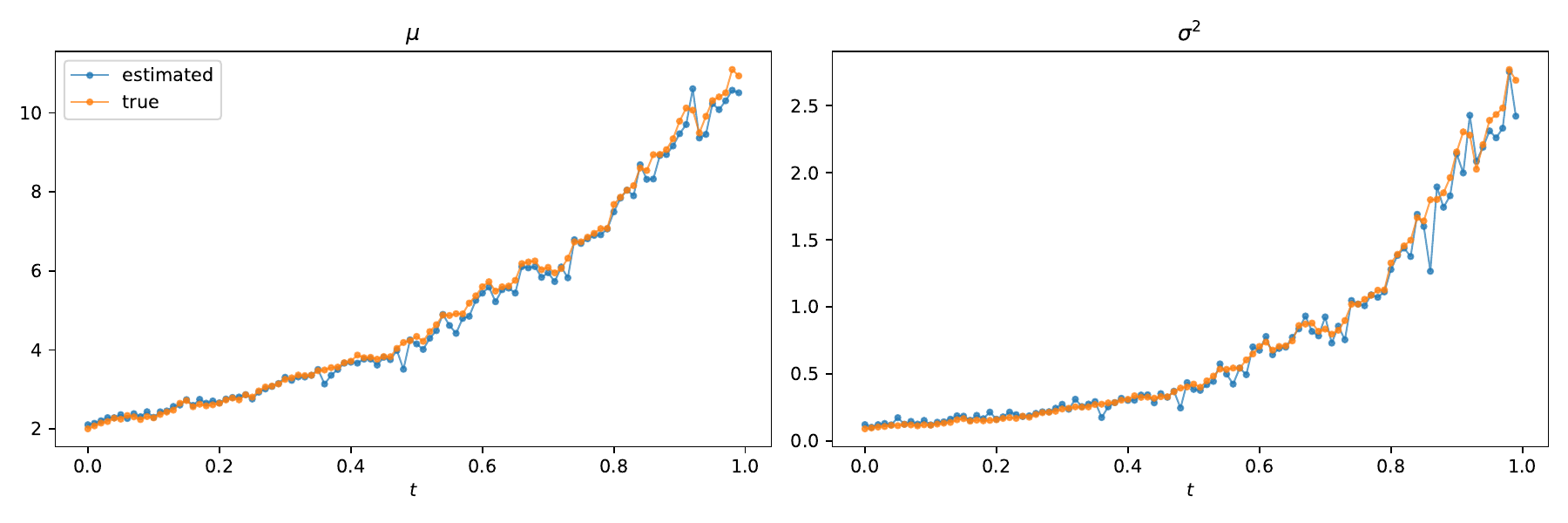}
\caption{True and estimated (with joint instantaneous method) drift and diffusion coefficients along one generated path.}
\label{fig:GBM coeffs}
\end{figure}

\begin{figure}
\centering
\includegraphics[width=0.6\textwidth]{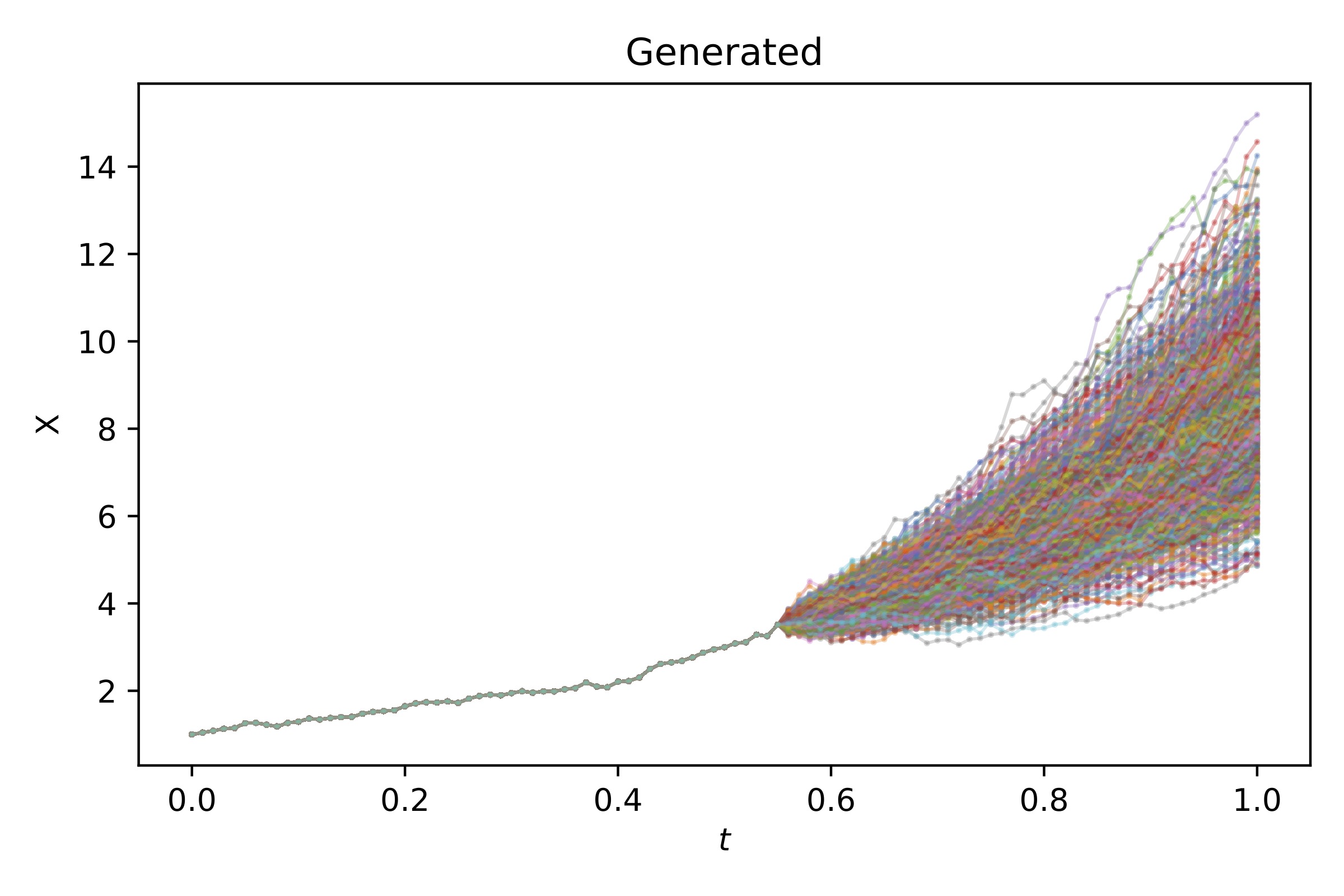}
\caption{$1000$ generated (with joint instantaneous method) path continuations, starting from the history of the first training path until $t=0.55$.}
\label{fig:GBM gen paths starting from sequence}
\end{figure}

\subsubsection{1-Step Ahead Training}\label{sec:1-step Ahead Training}
Even though the training for long-term predictions is recommended (cf.~\Cref{rem:long-term prediction suggested}), the generative method also works quite well without it in the case of complete regular observations. Here, we use the same dataset as before, however, with observation probability $p=1$ instead of $p=0.1$ (used before) and train without the learning approach for long-term predictions. We used the joint instantaneous coefficient estimation method. The results are shown in \Cref{tab:results GBM comparison}, named \textbf{(Joint Instant 1-step)}. We see that this training leads to results similar to those of the standard joint instantaneous method, outperforming all other methods\footnote{%
We note that training with the long-term prediction approach on this dataset should lead to better results than the joint instantaneous method, since it has roughly $10$ times as much training data available.}. In particular, we do not see small short-term errors blowing up over longer time periods as discussed in \Cref{rem:long-term prediction suggested}, which is a side effect of the joint instantaneous training that leads to very high accuracy in the instantaneous parameter predictions.

\subsection{Ornstein-Uhlenbeck Process}\label{sec:Ornstein-Uhlenbeck Process}
We consider a one-dimensional Ornstein-Uhlenbeck (OU) process satisfying the SDE
\begin{equation}\label{equ:OU}
    \dd X_t = \kappa (\theta - X_t)\dd t + \sigma \dd W_t,
\end{equation}
where $W$ is a Brownian motion, $\kappa > 0$ is the speed of reversion to the mean, $\theta \in \R$ is the long-term mean of the process, and $\sigma > 0$ is the volatility. We use the parameters $\kappa = 2, \theta=3, \sigma=1$ and set $X_0 = 1$, which leads to a growth towards $\theta=3$ (in mean). Based on the results of \Cref{sec:Geometric Brownian Motion}, we only report results for the joint instantaneous parameter estimation methods.
Similarly as for the GBM case, we estimate the parameters of the OU model (see \Cref{sec:Estimation of the Ornstein-Uhlenbeck Parameters} for the description of the estimation method) on the generated samples and on the training set and compare those to the true parameters in \Cref{tab:results OU comparison}.
Moreover, we plot $1000$ paths of the training set and the generated samples each in \Cref{fig:OU gen paths} and show the comparison of the marginal distributions of the true and generated values $X_t$ for $t=T/2=0.5$ and $t=T=1$ in \Cref{fig:OU distribution}.

\begin{table}[tb]
\caption{Ornstein-Uhlenbeck parameters $\kappa, \theta, \sigma$ estimated (see \Cref{sec:Estimation of the Ornstein-Uhlenbeck Parameters} for estimation method) on the generated samples and on the training dataset as reference.}
\label{tab:results OU comparison}
\begin{center}
\begin{tabular}{l | r r r }
\toprule
    & $\kappa$ & $\theta$ & $\sigma$ \\
\midrule[\heavyrulewidth]
True params & 2.0 & 3.0 & 1.0 \\
Reference & 2.0213 & 3.0060 & 1.0091 \\
\midrule
Joint Instant & 2.1642 & 3.0216 & 1.0293 \\
\bottomrule
\end{tabular}
\end{center}
\end{table}

\begin{figure}
\centering
\includegraphics[width=1\linewidth]{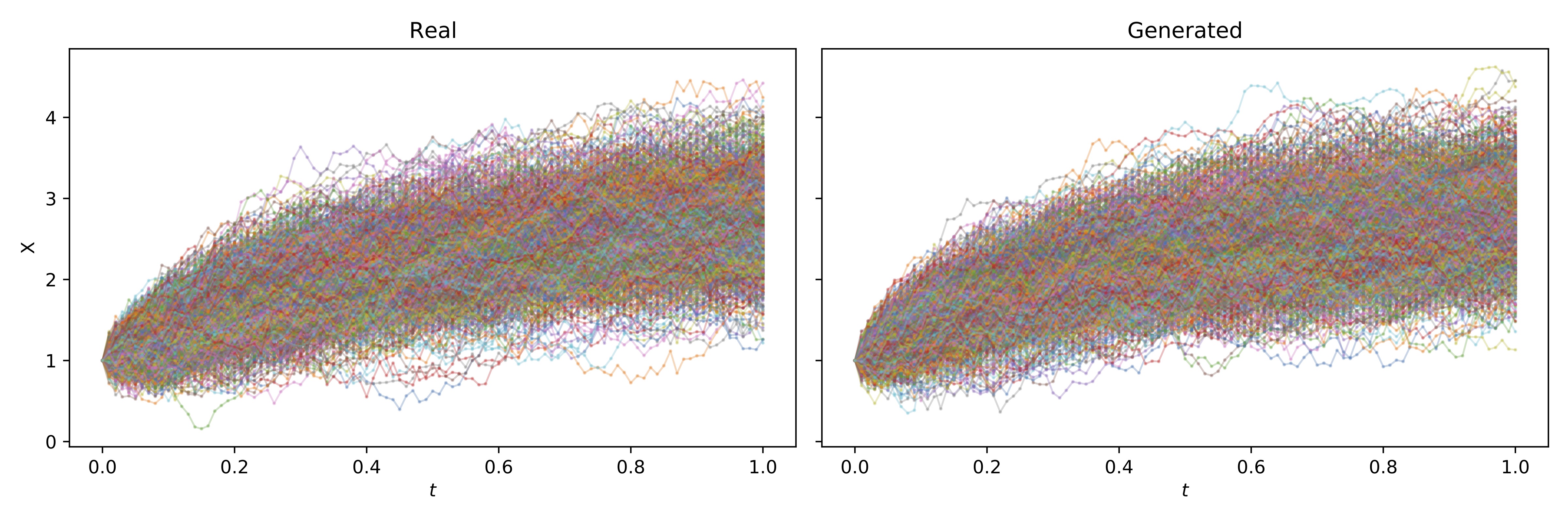}
\caption{Plot of true training paths and generated (with joint instantaneous method) paths, with $1000$ samples each.}
\label{fig:OU gen paths}
\end{figure}

\begin{figure}
\centering
\includegraphics[width=0.49\textwidth]{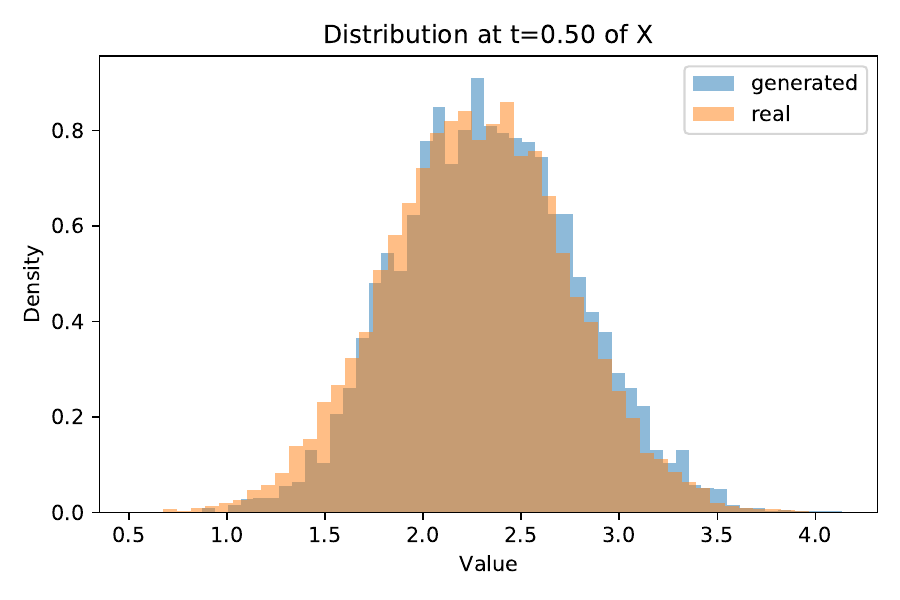}
\includegraphics[width=0.49\textwidth]{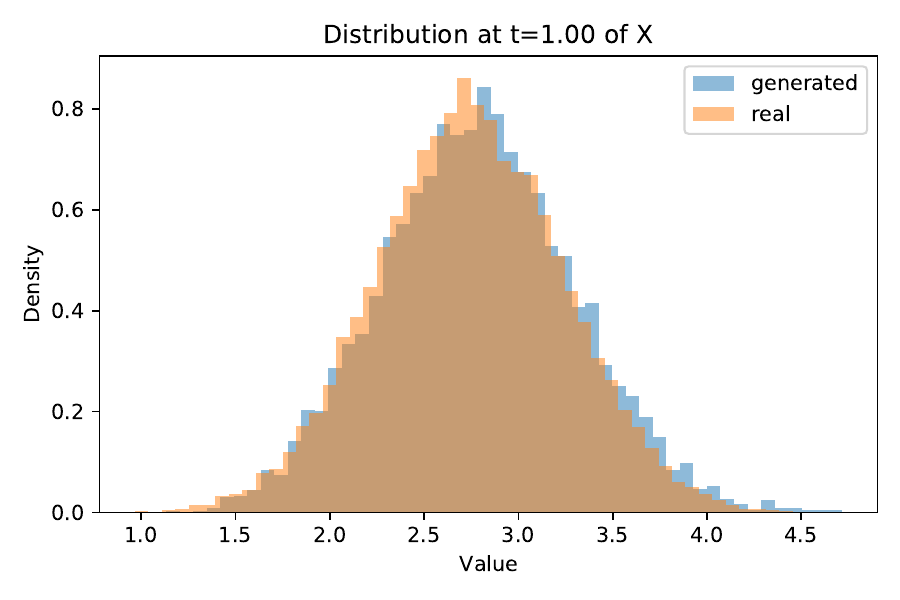}
\caption{Distribution of $X_t$ at $t=0.5$ and $t=T=1$ of true training paths and generated (with joint instantaneous method) paths.}
\label{fig:OU distribution}
\end{figure}


\if\addackn1
	\section*{Acknowledgement}
 
\fi

\bibliographystyle{iclr2021_conference}
\bibliography{references.bib}

\if\inclapp1
	\clearpage
	\appendix
    \section*{Appendix}

\fi


\appendix

\section{Applying the NJODE in the Generative Setting}\label{sec:Applying the NJODE in the Generative Setting}

To apply the NJODE, we need the main convergence results \citet[Theorems~4.1 and~4.]{krach2025operator}. In particular, we need to show that Assumptions~1 to~7 of \citet{krach2025operator} are satisfied in our setting through \Cref{ass:1,assumption:2,assumption:3,assumption:4}.
We first recall Assumptions~1 to~7 of \citet{krach2025operator} adjusted for our setting (since we are in the case $|\Xi|=1$, the assumptions simplify and  Assumption~6 can be dropped entirely) and then prove \Cref{prop:NJODE assumptions satisfied}.

\begin{AssumptionAppendix} \label{AssumptionAppendix:1}
For every $1\leq k, l \leq K$, $M_k$ is independent of  $t_l$and $n$, and 
$ \P (M_{k,i} =1 ) > 0$ for every component  $1 \leq i \leq d_X$ of the vector  (every component can be observed at any observation time and point).
\end{AssumptionAppendix} 
\begin{AssumptionAppendix} \label{AssumptionAppendix:3}
Almost surely $X$ is  not observed at a jump, i.e., $\P( \Delta X_{t_i} \neq 0 | i\leq n) = 0$ for all $ 1 \leq i \leq \bar n$.
\end{AssumptionAppendix} 
\begin{AssumptionAppendix} \label{AssumptionAppendix:4}
We assume that $F^X, F^Z$ are measurable and that there exist measurable functions $f^X, f^Z : [0,T]  \times (\R^{d})^{\N} \to \R^{d_X}$, generalized derivatives of $F^X,F^Z$, respectively, such that for all $t \in [0,T]$ and $(f,F) \in \{ (f^X,F^X),(f^Z,F^Z)\}$,
\begin{equation*}
F(t,O_{[0,\tau(t)]}) = F(\tau(t),O_{[0,\tau(t)]}) + \int_{\tau(t)}^t f(s,O_{[0,\tau(t)]}) ds.
\end{equation*}
Moreover, we assume that
\begin{equation}\label{equ:AssumptionAppendix4 bound}
\E\left[ \frac{1}{n} \sum_{i=1}^n  \Big(|F(t_i, O_{[0,t_i]} )|_2^2 + |F(t_{i-1},  O_{[0,t_{i-1}]} )|_2^2 
+ \int_0^T | f(t, O_{[0,\tau(t)]} )  |_2^2 dt  \Big) \right] < \infty.
\end{equation}
\end{AssumptionAppendix} 
\begin{AssumptionAppendix} \label{AssumptionAppendix:5}
We assume square integrability at observations $\displaystyle \E\left[\frac{1}{n}  \sum_{i=1}^n |X_{t_i}|_2^2 \right] < \infty$.
\end{AssumptionAppendix} 
\begin{AssumptionAppendix} \label{AssumptionAppendix:6}
The random number of observation times $n$ is integrable, i.e., $\E[n] < \infty$.
\end{AssumptionAppendix} 
\setcounter{theorem}{6}
\begin{AssumptionAppendix} \label{AssumptionAppendix:8}
The process $X$ is independent of the observation framework, i.e., of the random variables $n, (t_k, M_k)_{k \in \N}$. 
\end{AssumptionAppendix} 

\begin{proof}[Proof of \Cref{prop:NJODE assumptions satisfied}]
    First note that we are in the original setting of \citet{krach2022optimal}, i.e., in the setting $|\Xi|=1$ as in \citet[Remark 2.1]{krach2025operator}. Therefore, our \Cref{assumption:2,assumption:3,assumption:4} directly imply that Assumptions 1, 5, 6 and 7 of \citet{krach2025operator} are satisfied. Moreover, since $X$ is continuous by definition, Assumption 2 is satisfied. 
    $Z$ is continuous except for jumps at observations, where the left and right limits are observed, which we can deal with using \citet[Remark~2.4]{krach2025operator}.
    The uniform boundedness of $\mu, \sigma$, say by a constant $M$, implies integrability, since from \eqref{eq:def X} we get $|X_t| \leq |x_0| + M t + M |W_t| $. Since all moments of $W_t \sim N(0,t)$ are finite, all moments of $X_t$ are finite, hence, Assumption 4 is satisfied for $X$ and $Z$.
    
    Finally, we show Assumption 3.
    For $X$, note that the function $F$ is measurable and that we can us \eqref{eq:def X} to write for $s=\tau(t)$,
    \begin{equation*}
    \begin{split}
        \E[X_t | \mathcal{A}_s] &= \E[X_s | \mathcal{A}_s] + \E \left[ \int_s^t \mu_r(X_{\cdot \wedge r})\dd r \mid \mathcal{A}_s \right] + \E \left[ \int_s^t \sigma_r(X_{\cdot \wedge r})\dd W_r \mid \mathcal{A}_s \right] \\
        &= \E[X_s | \mathcal{A}_s] + \int_s^t \E \left[  \mu_r(X_{\cdot \wedge r})\mid \mathcal{A}_s \right] \dd r ,
    \end{split}
    \end{equation*}
    using Fubini's theorem (for conditional expectations) and that the integral with respect to $\dd W_r$ is a martingale.
    Measurability of the function $f^X(r,O_{[0,s]}) = \E \left[  \mu_r(X_{\cdot \wedge r})\mid \boldsymbol{\sigma}(O_{[0,s]}) \right]$ follows from continuity of $\mu$ and a similar argument as for $F^X$. Moreover, boundedness of $\mu$ implies that all powers of $f^X$ are integrable and since all moments of $X$ are finite, also the powers of $F^X$ are integrable (by Jensen's inequality). Hence, Assumption 3 holds for $X$. Next we use It\^o's formula to rewrite $Z$ for $\tau(t) \leq t \leq t_{\kappa(t)+1}$ as 
    \begin{equation*}
    \begin{split}
        Z_{t} &= \int_{\tau(t)}^t 2 (X_s - X_{\tau(t)}) \dd X_s^\top + \int_{\tau(t)}^t \dd [X_s, X_s^\top] \\
        &= \int_{\tau(r)}^t 2 (X_s - X_{\tau(t)}) \mu_s^\top \dd s + \int_{\tau(t)}^t 2 (X_s - X_{\tau(t)}) (\sigma_s \dd W_s)^\top  + \int_{\tau(t)}^t \Sigma_s \dd s.
    \end{split}
    \end{equation*}
    Similarly as before for $X$, we have that 
    \begin{equation*}
        \E[Z_t | \mathcal{A}_{\tau(t)} ] = \int_{\tau(t)}^t \E \left[ 2 (X_s - X_{\tau(t)}) \mu_s^\top + \Sigma_s \mid \mathcal{A}_{\tau(t)} \right] \dd s,
    \end{equation*}
    where we used that the integral with respect to $\dd W_s$ is a martingale by \citet[Lemma before Thm. 28, Chap. IV]{Pro1992}, using integrability of $X$ and boundedness of $\sigma$.
    Now we can conclude that Assumption 3 holds for $Z$ similarly as before for $X$, again using integrability of $X$ and boundedness and continuity of $\mu, \sigma$.
\end{proof}

\section{Details for Implementation}\label{sec:Details for implementation}

\subsection{Differences between the Implementation and the Theoretical Description of the NJODE}\label{sec:Differences between the Implementation and the Theoretical Description of the PD-NJODE}
Since we basically use the same implementation of the NJODE, all differences between the implementation and the theoretical description listed in \citet[Appendix~D.1.1]{krach2022optimal} also apply here.

\subsection{Details for Synthetic Datasets}\label{sec:Details for Synthetic Datasets}
Below we list the standard settings for all synthetic datasets. Any deviations or additions are listed in the respective subsections of the specific datasets.

\paragraph{Dataset} 
We use the Euler scheme to sample paths from the given stochastic processes on the interval $[0,1]$, i.e., with ${T}=1$ and a discretisation time grid with step size $0.01$ leading to $101$ grid points. At each time point we observe the process with probability $p=0.1$. We sample $20'000$ paths of which $80\%$ are used as training set and the remaining $20\%$ as validation set.

\paragraph{Architecture}
We use the NJODE with the following architecture. The latent dimension is $d_H = 100$ and all 3 neural networks have the same structure of 1 hidden layer with $\operatorname{ReLU}$ activation function and $50$ nodes. The signature is not used, the encoder is recurrent and the both the encoder and decoder use a residual connection. The inputs to the neural ODE are not scaled.

\paragraph{Training}
We use the Adam optimizer with the standard choices $\beta = (0.9, 0.999)$, weight decay of $0.0005$ and learning rate $0.001$. Moreover, a dropout rate of $0.1$ is used for every layer and training is performed with a mini-batch size of $200$ for $200$ epochs.
The NJODE models are either trained with the loss function \eqref{equ:Psi} or with \eqref{equ:Psi noisy}, depending on whether the baseline or the instantaneous estimators are learned. The model's diffusion output $G^\theta_2$ is squared to obtain $S^\theta = G^\theta_2 (G^\theta_2)^\top$, which is passed to the respective loss function. For learning the process $Z$, we use $Z_{t_i} = 0$ as additional input at observation times $t_i$, which ensures easier learning of the jumps to $0$.
In the baseline training, we do not use the long-term prediction training since we already train on a dataset with very irregular, and only a few, observations per sample, which has the same effect.

\paragraph{Model selection via early stopping}
We report the results for the best early stopped model, selected based on the validation loss. For some models, we only allow for early stopping after $100$ epochs, if they would otherwise stop before epoch $90$.

\subsubsection{GBM 1-step ahead training}
\paragraph{Dataset} 
The dataset is generated as detailed before, but with observation probability $p=1$, meaning that all $101$ grid points are observed for all samples.

\paragraph{Training} 
For the purpose of this analysis, we do not train with the long-term prediction method, which would be recommended for dense observations.

\section{Estimation of the Ornstein-Uhlenbeck Parameters}\label{sec:Estimation of the Ornstein-Uhlenbeck Parameters}
Given a set of $N \in \N$ independent path realisations 
of an OU process $X$ \eqref{equ:OU}, which is observed on a regular grid with $\nu+1 \in \N$ grid points (for simplicity, assumed to be indexed by integers, $(X_t)_{0\leq t \leq \nu}$), we can estimate the corresponding parameters of the OU process as follows. 
First, we note that the solution of the SDE \eqref{equ:OU} can be written in closed form for $s < t$ and $\Delta = t-s$ as
\begin{equation*}\label{equ:OU solution}
    X_t = X_s e^{-\kappa \Delta} + \theta (1 - e^{-\kappa \Delta}) + \frac{ \sigma \sqrt{1-e^{-2\kappa \Delta}}}{\sqrt{2 \kappa}} \epsilon,
\end{equation*}
where $\epsilon \sim N(0,1)$ is a standard normal random variable\footnote{More precisely this is a weak formulation of the solution, while the strong formulation holds for $s=0$ upon replacing $\sqrt{1-e^{-2\kappa \Delta}} \epsilon$ by $W_{1-e^{-2\kappa \Delta}}$.}.
Then we fit the parameters $\alpha, \beta$ of a linear regression model that regresses the next value of $X$ on the current one, i.e.,
\begin{equation*}
    X_{t+1} = \alpha + \beta X_t + \tilde\epsilon,
\end{equation*}
using all $\nu$ pairs of consecutive observations $(X_t, X_{t+1})$ of all $N$ paths. From \eqref{equ:OU solution} we infer that the regression parameters have to satisfy
\begin{equation*}
    \beta = e^{-\kappa \Delta}, \quad 
    \alpha = \theta (1 - e^{-\kappa \Delta})
\end{equation*}
and that the residuals $\tilde \epsilon = X_{t+1} - (\alpha + \beta X_t)$ have the variance 
$$\operatorname{Var}(\tilde \epsilon) = \frac{ \sigma^2 (1-e^{-2\kappa \Delta})}{2 \kappa}.$$
Hence, we can compute the OU parameters as 
\begin{equation*}
    \kappa = \frac{- \log(\beta)}{\Delta}, \quad \theta = \frac{\alpha}{1-\beta}, \quad \sigma = s \frac{\sqrt{2 \kappa}}{\sqrt{1-\beta^2}},
\end{equation*}
where $s = \sqrt{\operatorname{Var}(\tilde \epsilon)}$ is the standard deviation of the residuals.

\end{document}